%% file: paper-OPRE.tex
\documentclass[opre,sglanonrev]{informs4}

\RequirePackage{tgtermes}
\RequirePackage{newtxtext}
\RequirePackage{newtxmath}
\RequirePackage{bm}
\RequirePackage{endnotes}

\OneAndAHalfSpacedXII 

\usepackage{algorithm}
\usepackage{algpseudocode}
\usepackage{tikz}
\usetikzlibrary{decorations.text,shapes,snakes,arrows.meta}

\usepackage{natbib}
 \bibpunct[, ]{(}{)}{,}{a}{}{,}%
 \def\bibfont{\small}%

\newcommand{\br}{\bm{r}}
\newcommand{\bR}{\bm{R}}
\newcommand{\bz}{\bm{z}}
\newcommand{\by}{\bm{y}}
\newcommand{\E}{\operatorname{\mathbb{E}}}
\newcommand{\R}{\mathbb{R}}
\newcommand{\bx}{\bm{x}}
\newcommand{\util}{U}
\newcommand{\N}{\mathcal{N}}
\newcommand{\Si}{\mathcal{S}_{\textnormal{IM}}}
\DeclareMathOperator{\TV}{TV}
\newcommand{\goto}{\rightarrow}
\renewcommand\d{\operatorname{\mathrm{d}}}
\newcommand{\ba}{\bm{a}}
\newcommand{\bb}{\bm{b}}
\newcommand{\bI}{\bm{I}}
\newcommand{\ind}{\mathtt{ind}}
\newcommand{\I}{\mathcal{I}}
\newcommand{\succeqw}{\succeq_{\textnormal{w}}}
\newcommand{\succeqn}{\succeq_{\textnormal{no}}}
\newcommand{\bzz}{\bm{c}}
\newcommand{\zz}{c}
\newcommand{\bv}{\bm{v}}
\newcommand{\bH}{\bm{H}}
\newcommand{\bu}{\bm{u}}
\newcommand{\e}{\mathrm{e}}

\newtheorem{othertheorem}{Theorem}[section]

\EquationsNumberedThrough    

\TheoremsNumberedThrough     
\ECRepeatTheorems  %

\MANUSCRIPTNO{MOOR-0001-2024.00}

\begin{document}

\newcommand\blfootnote[1]{%
  \begingroup
  \renewcommand\thefootnote{}\footnote{#1}%
  \addtocounter{footnote}{-1}%
  \endgroup
}


\RUNAUTHOR{Su}

\RUNTITLE{The Isotonic Mechanism}

\TITLE{You Are the Best Reviewer of Your Own Papers: The Isotonic Mechanism}

\ARTICLEAUTHORS{%
\AUTHOR{Weijie Su}
\AFF{Department of Statistics and Data Science,
University of Pennsylvania, \EMAIL{suw@wharton.upenn.edu}}

} 

\ABSTRACT{%
Machine learning (ML) and artificial intelligence (AI) conferences including NeurIPS and ICML have experienced a significant decline in peer review quality in recent years. To address this growing challenge, we introduce the Isotonic Mechanism, a computationally efficient approach to enhancing the accuracy of noisy review scores by incorporating authors' private assessments of their submissions. Under this mechanism, authors with multiple submissions are required to rank their papers in descending order of perceived quality. Subsequently, the raw review scores are calibrated based on this ranking to produce adjusted scores. We prove that authors are incentivized to truthfully report their rankings because doing so maximizes their expected utility, modeled as an additive convex function over the adjusted scores. Moreover, the adjusted scores are shown to be more accurate than the raw scores, with improvements being particularly significant when the noise level is high and the author has many submissions---a scenario increasingly prevalent at large-scale ML/AI conferences.

We further investigate whether submission quality information beyond a simple ranking can be truthfully elicited from authors. We establish that a necessary condition for truthful elicitation is that the mechanism be based on pairwise comparisons of the author's submissions. This result underscores the optimality of the Isotonic Mechanism, as it elicits the most fine-grained truthful information among all mechanisms we consider. We then present several extensions, including a demonstration that the mechanism maintains truthfulness even when authors have only partial rather than complete information about their submission quality. Finally, we discuss future research directions, focusing on the practical implementation of the mechanism and the further development of a theoretical framework inspired by our mechanism.
}%

\FUNDING{This research was supported by NSF grants CCF-1934876 and CAREER DMS-1847415, and an Alfred Sloan Research Fellowship.}




\maketitle


\blfootnote{A preliminary version of some results of the paper was presented at NeurIPS 2021 \citep{su2021you}.}

\input{intro}

\input{results}

\input{compare}

\input{coarse}

\input{extend}

\input{proofs}

\input{discuss}



%
%
%

\ACKNOWLEDGMENT{We are grateful to two anonymous referees whose constructive comments helped substantially improve the presentation of this paper. We would like to thank Patrick Chao, Tan Gan, Qiyang Han, Nihar Shah, Haifeng Xu, Rakesh Vohra, Yuhao Wang, and Xingtan Zhang for very insightful comments and fruitful discussions.}



\bibliographystyle{informs2014} 
\bibliography{ref} 




\clearpage

\input{appendix}

\end{document}

%% file: intro.tex
\section{Introduction}
\label{sec:introduction}

The past decade has seen unprecedented attention to and impact from machine learning (ML) and artificial intelligence (AI). This technological revolution has been accompanied by a dramatic surge in submissions to premier ML/AI conferences such as NeurIPS, ICML, and ICLR. For example, NIPS\footnote{The name has changed from NIPS to NeurIPS since 2019.} received 833 submissions in 2006 and 1,838 in 2015, one year before AlphaGo was introduced, with an average annual growth of 9.2\% \citep{papercopilot2025statsneurips}. From 2016 to 2024, however, the number of submissions increased from 2,425 to 15,671, with an average annual growth of 26.3\%. Similarly, ICML, the second largest ML/AI conference, saw its submissions rise from 6,538 in 2023 to 9,653 in 2024, marking a 47.6\% increase just in one year \citep{papercopilot2025statsicml}.

This unprecedented growth in the scale of ML/AI conferences has made it increasingly difficult to maintain review quality, primarily because the pool of qualified reviewers---those, for instance, who have published at least one paper in a premier ML/AI conference---cannot keep pace with the surge in submission volumes~\citep{sculley2018avoiding}. The resulting disparity places a significant burden on these reviewers, thus reducing the average time they can devote to each paper, and, more troublingly, prompting program chairs to recruit novice reviewers for top-tier ML/AI conferences~\citep{stelmakh2020novice}. As a result, the scores\footnote{Unlike journal reviews, conference reviewers typically assign a numeric score to each submission. For instance, NeurIPS uses a scale from 1 to 10.} provided by reviewers---perhaps the most influential factor in deciding acceptance or rejection---display surprisingly high variability and arbitrariness~\citep{stelmakh2020novice,redditors2021neurips,cortes2021inconsistency}. For example, in the NIPS 2014 experiment, if a paper was reviewed twice, about 57\% of the accepted papers would have been rejected. While some degree of review arbitrariness may reflect the diversity of scientific perspectives, such a high level as documented in the NIPS experiment is likely to allow many low-quality papers to be accepted and, conversely, many strong papers to be rejected. In the long run, this poor reliability of peer review could diminish public trust in the proceedings of ML/AI conferences, a risk the scientific community cannot afford to overlook~\citep{rogers2020can}.

In response, numerous efforts have sought to improve peer review for ML/AI conferences~\citep{van1999effect,arous2021peer,jecmen2020mitigating,liang2024large}, particularly in producing more reliable review scores that truly reflect submission quality. A common theme in this body of work is to design better questions for reviewers, provide incentives for them to invest more time, and refine paper-reviewer assignments~\citep{kobren2019paper,wang2018your,wang2020debiasing,leyton2024matching,xu2023one}. Notably, research on improving peer review has largely neglected the incorporation of authors' opinions in the evaluation process at ML/AI conferences. Unlike many conference reviewers, who often have limited research experience~\citep{shah2018design}, authors are arguably among the most qualified individuals to assess the quality of their own submissions, given their intimate familiarity with the work. Peer review quality could be substantially enhanced if authors' private assessments regarding their submission quality were incorporated into the decision-making process. However, when solicited, a key challenge is that authors have a strong incentive to report inflated scores to maximize their chances of acceptance.

\subsection{An owner-assisted mechanism}
In this paper, we investigate the truthful elicitation of private information from authors to derive enhanced scores for submissions. We begin by considering an author who submits $n \ge 2$ papers to an ML/AI conference. It has become increasingly common for a single author to submit multiple papers to the same conference~\citep{iclr}. For $1 \le i \le n$, the raw (average) review score for the $i^{\textnormal{th}}$ paper is modeled as\footnote{The average score is computed over typically three or four reviews per submission.}
\begin{equation}\label{eq:y}
y_i = R_i + z_i,
\end{equation}
where $\bR = (R_1, \ldots, R_n)$ represents the ground-truth scores of the papers and $\bz = (z_1, \ldots, z_n)$ denotes the noise vector. Suppose the author has full or partial knowledge of the ground truth $\bR$, and the objective is to improve its estimation using the raw scores $\by = (y_1, \ldots, y_n)$ together with the author's private information regarding $\bR$.

We propose a mechanism that requires an author to provide a ranking $\pi$ of the ground-truth vector $(R_1, \ldots, R_n)$, which orders her submissions in descending order of quality. This ranking must be submitted before the author is exposed to the review scores. Given this ranking, our mechanism computes adjusted scores $\widehat\bR^{\pi} = (\widehat R_1^{\pi}, \ldots, \widehat R_n^{\pi})$ by solving the following convex optimization problem:
\begin{equation}\label{eq:isotone_explicit}
\begin{aligned}
&\min_{\br} & ~~  &  \| \by - \br\|^2 \\
&\text{s.t.~} & ~~ & r_{\pi(1)} \ge r_{\pi(2)} \ge \cdots \ge r_{\pi(n)}.
\end{aligned}
\end{equation}

Computationally, this problem is an instance of isotonic regression---hence the name, the Isotonic Mechanism---and can be solved efficiently by the pool adjacent violators algorithm~\citep{kruskal64,barlow72,bogdan2015slope}.

We assume the author's objective is to maximize her payoff, which is a function of the solution $\widehat{\bR}^{\pi}$ to \eqref{eq:isotone_explicit}. Concretely, let $U$ be any nondecreasing convex function. The author is then modeled as acting rationally to maximize the expected overall utility $\E \bigl[ U(\widehat R_1^{\pi}) + \cdots + U(\widehat R_n^{\pi}) \bigr]$ by choosing among the $n!$ possible rankings $\pi$, whether truthfully or not.

Our main result demonstrates that, under the Isotonic Mechanism, the author's optimal strategy is to \textit{truthfully} report the ranking $\pi^\star$ that satisfies $R_{\pi^\star(1)} \ge R_{\pi^\star(2)} \ge \cdots \ge R_{\pi^\star(n)}$. 
\begin{othertheorem}[full statement in Theorem~\ref{thm:main} in Section~\ref{sec:conv-util-funct}]\label{thm:intro_all}
The author's expected utility is maximized when the Isotonic Mechanism is provided with the ground-truth ranking $\pi^\star$. 
\end{othertheorem}

A notable advantage of the Isotonic Mechanism is that, for an author to truthfully report a ranking of her submissions, it suffices for her to correctly perform pairwise comparisons between any two papers. From a practical standpoint, this is advantageous because pairwise comparisons do not require calibration; it is sufficient to determine the relative ordering of $R_1, \ldots, R_n$, even if the author is optimistic (i.e., perceiving the quality as $R_1+c, \ldots, R_n+c$ for some $c>0$) or pessimistic (i.e., perceiving the quality as $R_1+c, \ldots, R_n+c$ for some $c<0$). As long as the bias is consistent across all submissions, the author can still provide the correct ground-truth ranking.\footnote{For example, in the development of ChatGPT, pairwise comparisons were employed rather than requiring absolute scores for responses \citep{ouyang2022training}.}

From the perspective of the conference, the crucial consideration is whether the Isotonic Mechanism enhances the estimation of the ground-truth quality of submissions. Indeed, when the mechanism is supplied with the ground-truth ranking, the adjusted scores produced by the mechanism improve the estimation of $\bR$ in squared errors relative to the raw review scores $\by$ (see Proposition~\ref{thm:estimate}). Thus, the author's information about her own submissions contributes positively to enhancing the accuracy of $\bR$, and we therefore refer to this as an owner-assisted mechanism.

Moreover, our analysis shows that the benefit of the Isotonic Mechanism can become significant if the variability of the raw review scores is large and if the author has many submissions---precisely the scenario faced by modern ML/AI conferences. For example, 133 authors each submitted at least five papers to ICLR 2020, accounting for a notable fraction of the total 2,594 submissions. Surprisingly, a single individual submitted as many as 32 papers to that conference~\citep{iclr}. These trends make the Isotonic Mechanism especially well-suited to enhancing peer review in such contexts.

\subsection{Other truthful mechanism?}

Other than the Isotonic Mechanism, one might wonder if alternative truthful approaches exist. In particular, can one truthfully elicit more fine-grained information than a simple ranking? Intuitively, the more detailed the author-provided information (assuming it is truthful), the more effectively the conference can estimate the ground truth. Conversely, mechanisms that elicit coarser information may also be useful when the author is uncertain about the exact ranking of the ground truth, although this is likely to result in a less accurate estimation.

To formulate the problem of finding general truthful mechanisms in the owner-assisted setting, we consider an author who knows the ground truth $\bR \in \R^n$ and transmits a message about it. We assume this message concerns an unknown but fixed value of $\bR$---for instance, the message could state ``$R_i > c$ for all $i$'' or ``$R_1 + \cdots + R_n > c$'' for some constant $c$. Moreover, we assume that the message is uniquely determined by the ground truth; that is, there exists exactly one message that is consistent with $\bR$.

Let $\mathcal{M}$ denote the collection of all possible messages. For any message $\mathtt{Mess} \in \mathcal{M}$, let $S = S(\mathtt{Mess})$ denote the set of all values $\bR$ for which this message would be generated by the author. By construction, the family of sets $\mathcal{S} = \{S(\mathtt{Mess}): \mathtt{Mess} \in \mathcal{M}\}$ forms a ``knowledge partition'' of $\R^n$. The author is required to select an element from the partition $\mathcal{S}$ prior to viewing the review scores; that is, the message sent to the conference is essentially ``the ground-truth score vector of my submissions belongs to $S$''.

Given the author's selection $S$ from the partition and the raw review scores $\by$, the conference aims to accurately estimate $\bR$. Assuming that the conference accepts the author's message as truthful, a natural approach to incorporate the constraint $\bR \in S$ is to use $\by$ as the estimate if $\by \in S$, and otherwise project $\by$ onto $S$. Formally, the estimate is defined as the solution to the following optimization problem:
\begin{equation}\nonumber
\min_{\br} ~ \|\by - \br\|^2
\end{equation}
subject to the constraint that $\br = (r_1, \ldots, r_n) \in S$, where $\|\cdot\|$ denotes the $\ell_2$ norm throughout. This formulation corresponds to a constrained maximum likelihood estimation when the noise variables $z_1, \ldots, z_n$ in \eqref{eq:y} are independent and identically distributed (i.i.d.) normal random variables with mean zero~\citep{aitchison1958maximum}. This estimator is known to be minimax optimal under very general conditions~\citep{johnstone2002function}.

This estimation approach coincides with the Isotonic Mechanism when $\mathcal{S}$ is the \textit{isotonic partition}---obtained by partitioning the $n$-dimensional Euclidean space into $n!$ cones based on the ranking of the $n$ coordinates---which is truthful by Theorem~\ref{thm:intro_all}. However, for more general partitions, the author may have an incentive to misreport the set containing the ground truth, particularly when the partition is very fine-grained. For example, if the question is posed as ``What are the exact scores of your papers?'' and the partition is maximally fine-grained (i.e., each set consists of a single data point), the author may be incentivized to report inflated scores.

We provide below a necessary condition for a partition-based mechanism to be truthful. In the theorem statement, we consider an agnostic setting in which the utility is known only to be additive and convex across submissions. 

\begin{othertheorem}[full statement in Theorem~\ref{thm:compare} in Section~\ref{sec:honest-incl-inform}]\label{thm:intro1}
If the author always reports truthfully under the aforementioned assumptions on the mechanism and utility, then the partition $\mathcal{S}$ must be delineated by a set of pairwise-comparison hyperplanes of the form $x_i - x_j = 0$ for some $1 \le i < j \le n$.
\end{othertheorem}

This theorem implies that the author will be truthful \textit{only if} the elicitation question is based on \textit{pairwise comparisons}. To elaborate, a partition is delineated by pairwise-comparison hyperplanes if, for any element $S$ in the partition $\mathcal{S}$, membership $\bR \in S$ can be determined solely by verifying the inequalities $R_i \ge R_j$ for some indices $i < j$. For example, consider the collection $\{\bx \in \R^3: \min(x_i, x_{i+1}) \ge x_{i+2}\}$ for $i = 1, 2, 3$, which forms such a partition in three dimensions,\footnote{Overlaps along the boundaries are ignored as they constitute a Lebesgue measure zero set.} where we adopt the cyclic convention $x_{i+3} = x_i$. Specifically, note that $\bx \in S_i$ if $x_i \ge x_{i+2}$ and $x_{i+1} \ge x_{i+2}$. In contrast, consider the collection of spheres $\{\bx \in \R^n: \|\bx\| = c\}$ for all $c \ge 0$; such a collection cannot be generated solely by pairwise comparisons.

When all pairs of a vector are compared, one obtains the complete ranking of its coordinates, but no additional information. Hence, the most fine-grained partition achievable via pairwise comparisons is the isotonic partition, which is employed in the Isotonic Mechanism for information elicitation.

Taken together, Theorems~\ref{thm:intro_all} and \ref{thm:intro1} demonstrate the optimality of the Isotonic Mechanism. Specifically, this mechanism elicits the most fine-grained information concerning the ground truth among all partition-based truthful mechanisms. The necessary condition shown in Theorem~\ref{thm:intro1} indicates that any truthful partition must be derived from the isotonic partition, in the sense that each element of the partition must be a union of several rankings.

Furthermore, we identify several truthful partitions that, while based on pairwise comparisons, are strictly coarser than the isotonic partition (see Theorem~\ref{thm:block} in Section~\ref{sec:exampl-truth-tell}). This finding implies that when an author faces uncertainty about the ground-truth ranking, she can still maximize her utility by reporting a coarser form of the ranking to the best of her knowledge.

A practical advantage is that if the most fine-grained truthful partition were not defined via pairwise comparisons, implementing the mechanism would pose considerable practical challenges. For instance, it would be difficult for an author to estimate $\|\bR\|$, whereas performing pairwise comparisons is relatively straightforward---precisely the approach taken in the Isotonic Mechanism.

This practical benefit significantly facilitated our experiment at ICML in 2023, 2024, and 2025, where authors with multiple submissions were asked to rank their papers based on perceived quality. In ICML 2023, we collected 1,342 rankings corresponding to 2,592 submissions \citep{su2024analysis}. Our analysis indicated that scores adjusted by the Isotonic Mechanism were significantly more accurate than raw review scores under certain metrics~\citep{su2024analysis}.

\subsection{Related work}
\label{sec:related}

An extensive line of research has been proposed to address the declining quality of peer review in ML/AI conferences, predominantly focusing on the role of reviewers~\citep{van1999effect,arous2021peer,jecmen2020mitigating,liang2024large}. For instance, \citet{ugarov2023peer} introduced a mechanism that utilizes peer prediction to incentivize reviewers. In contrast, the Isotonic Mechanism exploits information exclusively from authors, without imposing any additional burden on reviewers. For completeness, we note that an emerging line of research in mechanism design leverages additional author information to improve peer review~\citep{aziz2019strategyproof,noothigattu2021loss,mattei2020peernomination}, albeit without explicitly soliciting rankings of authors' submissions.

From an economic perspective, the literature on cheap talk bears conceptual similarities to the present work \citep{crawford1982strategic,chakraborty2007comparative,chakraborty2010persuasion}. Roughly speaking, cheap talk refers to a class of games where the sender provides costless, non-binding information to the receiver, and this information does not directly affect the payoffs of either party. This literature demonstrates that when the interests of the two parties are aligned or partially aligned, an equilibrium can emerge in which the sender provides credible information, thereby yielding higher payoffs than in the absence of coordination. Moreover, \citet{battaglini2002multiple} demonstrated the advantages of multidimensional communication channels. A key distinction between cheap talk and our work, however, is that our mechanism explicitly commits to incorporating the ranking information into the estimation of ground-truth scores. For comparison, \citet{chakraborty2007comparative} examined a closely related setting by eliciting comparative information from the sender and primarily establishes the existence of truthful equilibria, whereas our work provides explicit computational procedures and establishes concrete statistical guarantees regarding improvements in estimation accuracy. For completeness, \citet{levy2007limits} demonstrated that commitment power in multidimensional cheap talk can enhance communication.

The present paper is also related to the literature on the delegation problem~\citep{holmstrom1978incentives,melumad1991communication,martimort2006continuity,amador2013theory}, in which a principal (in our case, the conference) delegates decision-making authority to an agent (the author) and commits to following the agent's decision, subject to a set of rules. The most relevant work to our study is the seminal contribution on aligned delegation~\citep{frankel2014aligned}, which examined how a principal can delegate multiple decisions to an agent with unknown preferences. In \citet{frankel2014aligned}, mechanisms are derived that ensure the principal maximizes its payoff in a robust sense, regardless of potential agent biases, with the agent effectively acting as if maximizing the principal's utility. In particular, simple delegation rules, such as requiring the specification of rankings or imposing fixed budgets, guarantee the best possible worst-case outcomes for the principal. Furthermore, \citet{frankel2016delegating} demonstrated that when delegating multiple decisions to a biased agent, it is optimal for the principal to employ a half-space delegation set that constrains the agent's choices against potential biases. In a broad sense, the Isotonic Mechanism can be interpreted within the framework of aligned delegation, as the conference delegates the task of ranking submissions to the author. A significant distinction, however, is that our mechanism focuses on its real-world application in ML/AI conference peer review, where the principal observes noisy signals of the ground truth, which enables the development of estimation procedures with provable statistical guarantees. Thus, the primary focus of the Isotonic Mechanism is statistical estimation rather than the establishment of truthful equilibria, with truthfulness serving merely as an intermediate step to achieve accurate estimation of the ground truth.

\subsection{Organization of the paper}
The remainder of the paper is organized as follows. In Section~\ref{sec:conv-util-funct}, we formally introduce the Isotonic Mechanism and establish the conditions under which the mechanism is truthfulness, accompanied by an analysis of its estimation properties. Section~\ref{sec:honest-incl-inform} examines general partition-based mechanisms and establishes the optimality of the Isotonic Mechanism. In Section~\ref{sec:exampl-truth-tell}, we demonstrate that while some pairwise-comparison-based partitions yield truthful mechanisms, others fail to maintain this property. Section~\ref{sec:extens-util-funct} presents several extensions showing that truthful reporting yields the highest payoff in more general settings. The proofs of the main results are presented in Section~\ref{sec:proofs}. Finally, Section~\ref{sec:discussion} concludes the paper with a discussion of potential directions for future research.


%% file: results.tex
\section{The Isotonic Mechanism}
\label{sec:conv-util-funct}

In this section, we present the Isotonic Mechanism in detail and establish theoretical guarantees under certain conditions. These guarantees show that (1) the author maximizes her expected utility by truthfully reporting the ranking of her submissions, and (2) this mechanism improves the statistical estimation accuracy of submission quality.

Consider an author who submits \(n \ge 2\) papers to an ML/AI conference. The raw review scores \(\by \in \R^n\) of these \(n\) submissions are assumed to follow the model $\by = \bR + \bz$, where \(\bR\) measures the quality of the submissions and can be viewed as the ground-truth scores of the \(n\) papers, and \(\bz\) denotes \(n\) noise terms. The Isotonic Mechanism requires the author to provide a ranking of her submissions in descending order of quality; let \(\pi\) denote this ranking. The ranking is provided before the raw review scores \(\by\) are disclosed to the author. Subsequently, the Isotonic Mechanism outputs \textit{ranking-calibrated} scores---denoted \(\widehat{\bR}^{\pi}\) as before---which are given by the solution to the following optimization problem:
\begin{equation}\label{eq:isotone}
\begin{aligned}
&\min_{\br} ~ \| \by - \br\|^2 \\
&\text{~s.t.} ~~ \br \in S_{\pi},
\end{aligned}
\end{equation}
where the isotonic cone \(S_{\pi}\) is defined as
\[
S_{\pi} = \left\{ \bx = (x_1, \ldots, x_n): x_{\pi(1)} \ge x_{\pi(2)} \ge \cdots \ge x_{\pi(n)} \right\}.
\]
Note that \(\bR \in S_{\pi^\star}\), where the ground-truth ranking \(\pi^\star\) sorts the components of \(\bR\) in descending order: $R_{\pi^\star(1)} \ge R_{\pi^\star(2)} \ge \cdots \ge R_{\pi^\star(n)}$. The collection of isotonic cones over all permutations forms the isotonic partition.\footnote{We omit the overlaps along the boundaries of the isotonic cones. For a treatment of general knowledge partitions, see Section~\ref{sec:honest-incl-inform}.}

We make the following assumptions to investigate under what conditions the Isotonic Mechanism incentivizes the author to report truthfully. In the language of mechanism design, the conference is the principal, who observes raw review scores as noisy signals of the ground truth, and the author is the agent who possesses private information about her submissions.

\begin{assumption}[Agent's private information]\label{ass:author2}
The author has sufficient knowledge of the ground truth \(\bR\) for her \(n \ge 2\) submissions to determine which isotonic cone contains \(\bR\).
\end{assumption}

\begin{remark}
This assumption is weaker than requiring the author to know the exact values of the ground-truth scores \(R_1, \ldots, R_n\). In particular, it suffices for the author to know the relative magnitudes of \(R_i\)'s. For example, this holds if the author is aware of a monotone transformation of these scores.
\end{remark}

\begin{assumption}[Observation for the principal]\label{ass:noise}
The noise variables \(\bz = (z_1, \ldots, z_n)\) are i.i.d.\ draws from some probability distribution.
\end{assumption}

\begin{remark}
The independence condition can be relaxed to exchangeability for most results in this paper, unless stated otherwise. That is, \((z_1, \ldots, z_n)\) has the same probability distribution as \((z_{\pi(1)}, \ldots, z_{\pi(n)})\) for any permutation \(\pi\). Note also that the noise distribution can have nonzero mean.

\end{remark}

Suppose the conference makes decisions using the scores \(\br = (r_1, \ldots, r_n)\). We assume that the author's utility is additive across submissions.

\begin{assumption}[Agent's utility]\label{ass:convex}
The author's overall utility has the form
\begin{equation}\label{eq:util_form}
\util\bigl(\br\bigr) := \sum_{i=1}^n U\bigl(r_i\bigr),
\end{equation}
where \(U\) is a nondecreasing convex function.
\end{assumption}

\begin{remark}
Sections~\ref{sec:true-grade-dependent} and~\ref{sec:extens-util-funct} discuss two relaxations of the convexity assumption.
\end{remark}

\begin{remark}
To put the convexity assumption differently, the marginal utility $U'$ is nondecreasing. Convex utility is often assumed in the literature \citep{krishna2001convex} and, in particular, does not contradict the economic law of diminishing marginal utility~\citep{kreps1990course}, as the score measures quality rather than quantity. In peer review for ML/AI conferences, for example, high scores largely determine whether an accepted paper is presented as a poster, an oral presentation, or granted a best paper award. While an oral presentation attracts more attention than a poster, a best paper award significantly enhances the paper's impact. Accordingly, the marginal utility tends to be larger when scores are higher. Another example is the diamond-quality grading system used by the Gemological Institute of America. Typically, the price of a high-grade diamond increases more rapidly with its grade than that of a low-grade diamond. Nevertheless, an interesting direction for future research is to provide more evidence for convex utility using empirical data.

\end{remark}

Rationally, the author chooses \(\pi\)---whether truthful or not---to maximize her expected overall utility \(\E\bigl[\util\bigl(\widehat{\bR}^{\pi}\bigr)\bigr]\) when the Isotonic Mechanism's ranking-calibrated scores are used in the conference's decisions.

\begin{definition}\label{def:truth}
The Isotonic Mechanism (or equivalently, the isotonic partition) is called \textit{truthful} with utility \(U\) if the author maximizes her expected overall utility by reporting the ground-truth ranking \(\pi^\star\).
\end{definition}

We now state the main result of this section.

\begin{theorem}\label{thm:main}
Under Assumptions~\ref{ass:author2}, \ref{ass:noise}, and~\ref{ass:convex}, the Isotonic Mechanism is truthful.
\end{theorem}

\noindent The proof of this theorem is given in Section~\ref{sec:xxxxxxxxxxxx}.

\begin{remark}
When there are no ties in the true scores \(\bR\) and \(U\) is strictly convex, reporting the ground-truth ranking is \textit{strictly} optimal: for any ranking \(\pi \neq \pi^\star\), we have
\[
\E\bigl[\util\bigl(\widehat{\bR}^{\pi}\bigr)\bigr] 
< 
\E\bigl[\util\bigl(\widehat{\bR}^{\pi^\star}\bigr)\bigr],
\]
where \(\widehat{\bR}^{\pi^\star}\) is the output of the mechanism under the ground-truth ranking.
\end{remark}

\begin{remark}
The squared \(\ell_2\) loss in~\eqref{eq:isotone} can be replaced by a sum of Bregman divergences~\citep{gneiting2011making}. If \(\phi\) is a twice continuously differentiable, strictly convex function, and \(D_{\phi}(y,r) = \phi(y) - \phi(r) - (y-r)\phi'(r)\) its associated Bregman divergence, then Theorem~\ref{thm:main} still holds if the conference uses the solution to
\begin{equation}\label{eq:isotone_shape_d}
\min_{\br \in S_{\pi}} ~ \sum_{i=1}^n D_{\phi}(y_i, r_i).
\end{equation}
Indeed, \eqref{eq:isotone} is the special case \(\phi(x) = x^2\). Another example is the Kullback--Leibler divergence \(D_{\phi}(y, r) = y\log\frac{y}{r} + (1-y)\log\frac{1-y}{1-r}\) for \(0 < y,r < 1\), where \(\phi(x) = x\log x + (1-x)\log(1-x)\). Minimizing \(\sum_{i=1}^n D_{\phi}(y_i, r_i)\) over the isotonic cone \(S_{\pi}\) yields the same solution as the Isotonic Mechanism for any continuously differentiable, strictly convex \(\phi\)~\citep{minimax}.
\end{remark}

To build intuition about why the Isotonic Mechanism is truthful, note a key property of isotonic regression. Beyond the mean-preserving constraint
\[
\sum_{i=1}^n \widehat{R}^{\pi}_{i} = \sum_{i=1}^n y_i,
\]
the solution to isotonic regression tends to exhibit smaller variability across its components if an incorrect ranking \(\pi\) is imposed instead of the ground-truth ranking~\citep{kruskal64}. Consequently, Jensen's inequality suggests that the overall convex utility \(\sum_{i=1}^n U\bigl(\widehat{R}^{\pi}_i\bigr)\) is usually smaller under an incorrect ranking.\footnote{Nevertheless, the proof of Theorem~\ref{thm:main} does not rely on Jensen's inequality.}

In the noiseless setting \(\by = \bR\), this phenomenon becomes clearer. When $\pi$ is truthfully set to $\pi^\star$, \(\bR\) itself is a feasible solution to~\eqref{eq:isotone}, so the ranking-calibrated scores \(\widehat{\bR}^{\pi^\star} = \bR\). However, if \(\pi \neq \pi^\star\), the pool adjacent violators algorithm will keep averaging certain components of \(\bR\) until \(\bR\) fits the (incorrect) ranking \(\pi\). Because the entries are then ``shrunk'' together, the sum \(\sum_{i=1}^n U\bigl(\widehat{R}^{\pi}_i\bigr)\) typically decreases owing to the convexity of \(U\).

In practice, the conference could commit to randomly choosing a ranking input if the author fails to provide one. Since doing so would yield a suboptimal outcome for the author, the author's best response in such a scenario would be to submit a ranking---and, as shown, the optimal choice is the ground-truth ranking.

\subsection{Estimation properties}
\label{sec:estim-prop}

Because the Isotonic Mechanism is truthful, it also improves estimation of the ground truth \(\bR\).

\begin{proposition}\label{thm:estimate}
Under Assumptions~\ref{ass:author2}, \ref{ass:noise}, and~\ref{ass:convex}, the Isotonic Mechanism improves estimation accuracy of \(\bR\) in the sense that
\[
\E \bigl\| \widehat{\bR}^{\pi^\star} - \bR \bigr\|^2 
\le 
\E \bigl\| \by - \bR \bigr\|^2.
\]
\end{proposition}

\begin{remark}
This improvement over the raw scores \(\by\) remains valid if \(S_{\pi}\) in~\eqref{eq:isotone} is replaced by any convex set containing \(\bR\). Although the proof is standard, we provide it in the Appendix for completeness.
\end{remark}

We next show that, for large \(n\) and noisy raw scores, the Isotonic Mechanism can significantly reduce the estimation error. Define the total variation of \(\bR\) by
\[
\TV(\bR) := \inf_{\pi}\sum_{i=1}^{n-1} \bigl|R_{\pi(i)} - R_{\pi(i+1)}\bigr|
= \sum_{i=1}^{n-1} \bigl|R_{\pi^\star(i)} - R_{\pi^\star(i+1)}\bigr| 
= R_{\pi^\star(1)} - R_{\pi^\star(n)}.
\]
Let \(z_1, \ldots, z_n\) be i.i.d.\ \(\N(0, \sigma^2)\). For any fixed \(\sigma > 0\) and \(V > 0\), if \(\TV(\bR) \le V\), then the Isotonic Mechanism with the true ranking \(\pi^\star\) satisfies
\[
0.4096 + o_n(1) 
\le 
\frac{\sup_{\TV(\bR) \le V} \E \bigl\| \widehat{\bR}^{\pi^\star} - \bR \bigr\|^2 }
     {n^{\tfrac13} \,\sigma^{\tfrac43} \,V^{\tfrac23}} 
\le 
7.5625 + o_n(1),
\]
where both \(o_n(1)\) terms tend to \(0\) as \(n \goto \infty\).

\begin{remark}
This result follows from known risk bounds for isotonic regression; see~\citet{zhang2002risk,chatterjee2015risk}.
\end{remark}

Hence, for suitably constrained \(\bR\), the risk of \(\widehat{\bR}^{\pi^\star}\) (in terms of squared error) is on the order of \(n^{1/3} \sigma^{4/3}\). In contrast, using \(\by\) directly entails a risk of
\[
\E \Bigl[\|\by - \bR\|^2\Bigr] 
= \E \Bigl[\sum_{i=1}^n z_i^2\Bigr] 
= n\,\sigma^2.
\]
The ratio between the two risks is $O(n^{1/3} \sigma^{4/3})/(n \sigma^2) = O(n^{-2/3} \sigma^{-2/3})$. Therefore, the Isotonic Mechanism is especially favorable when both $n$ and $\sigma$ are large. In interpreting this result, however, it is important to notice that the total variation of the ground truth is fixed. Otherwise, when $R_{\pi^\star(i)} \gg R_{\pi^\star(i+1)}$ for all $i$, the solution of the Isotonic Mechanism is roughly the same as the raw scores $\by$ because it satisfies the constraint $\by \in S_{\pi^\star}$ with high probability. Accordingly, the Isotonic Mechanism has a risk of about $n \sigma^2$ in this extreme case. That said, the Isotonic Mechanism in general is superior to using the raw scores, according to Proposition~\ref{thm:estimate}.

\subsection{Ground-truth-dependent utility}
\label{sec:true-grade-dependent}

The utility function may depend on the ground truth. For instance, an author might prefer a high-quality paper to receive a higher score rather than a low-quality paper. This is especially relevant when a paper is awarded a best paper prize. If a low-quality paper receives the award, its weaknesses or flaws may be exposed as more people scrutinize it closely~\citep{carlini2022no}.

To accommodate such scenarios, we relax Assumption~\ref{ass:convex} so that utility may vary with the ground-truth score. Suppose two papers have ground-truth scores \(R\) and \(R'\) with \(R > R'\). Given two scores \(r\) and \(r'\) where \(r > r'\), the author prefers awarding the higher score to the paper with the higher ground truth. Formally,
\begin{equation}\label{eq:u_true_higher}
U(r;R) + U(r';R') \ge U(r;R') + U(r';R).
\end{equation}
Rewriting \(r = r' + \Delta r\) and letting \(\Delta r \to 0+\) implies
\[
\frac{\d}{\d r}U\bigl(r;R\bigr) 
\ge
\frac{\d}{\d r}U\bigl(r;R'\bigr).
\]
This motivates the following assumption.

\begin{assumption}\label{ass:convex2}
Given scores $\br$ in decision-making, the author's overall utility is
\[
\util(\br) 
:= 
\sum_{i=1}^n U\bigl(r_i;R_i\bigr),
\]
where \(U(x;R)\) is convex in its first argument and satisfies
\begin{equation}\label{eq:ut_inc}
\frac{\d}{\d x}U\bigl(x;R\bigr) 
\ge 
\frac{\d}{\d x}U\bigl(x;R'\bigr)
\quad
\text{whenever } R > R'.
\end{equation}
\end{assumption}

The inequality in this assumption amounts to saying that the marginal utility increases with respect to the true score of the submission. By integrating both sides of the inequality $\frac{\d U(x; R)}{\d x} \ge \frac{\d U(x; R')}{\d x}$, we find that Assumption~\ref{ass:convex2} implies \eqref{eq:u_true_higher}.

An example of ground-truth-dependent utility takes the form $U(x; R) = g(R)h(x)$, where $g \ge 0$ is nondecreasing and $h$ is a nondecreasing convex function. Taking any nondecreasing $g_1,
\ldots, g_L \ge 0$ and nondecreasing convex $h_1, \ldots, h_L$, more generally, the following function
\[
U(x; R) = g_1(R)h_1(x) + g_2(R)h_2(x) + \cdots + g_L(R)h_L(x)
\]
satisfies Assumption~\ref{ass:convex2}.

Theorem~\ref{thm:main} remains true in the presence of heterogeneity in the author's utility, as we show below. It is proved in Section~\ref{sec:prop:score}.

\begin{theorem}\label{prop:score_dependent}
Under Assumptions~\ref{ass:author2}, \ref{ass:noise}, and \ref{ass:convex2}, the Isotonic Mechanism is truthful. 

\end{theorem}

\begin{remark}

This result contrasts with the ranking mechanism in the aligned delegation literature~\citep{frankel2014aligned}, where an agent ranks items, and the mechanism assigns predetermined values based on the ranking. While such mechanisms ensure truthful ranking for any agent with increasing-difference utility \eqref{eq:ut_inc}---without requiring convexity---the Isotonic Mechanism, despite relying on stronger assumptions, has the advantage of producing estimates that are not fixed in advance but instead are adaptive to both the ranking and the raw scores. Moreover, the technical proofs underlying these mechanisms differ substantially.

\end{remark}


%% file: compare.tex
\section{When is truthfulness possible?}
\label{sec:honest-incl-inform}

In this section, we broaden our scope by considering a general class of mechanisms, including the Isotonic Mechanism as a special case, that can potentially elicit truthful information about the ground truth for estimation. The action space for the author is represented by a partition $\mathcal{S} := \{S_{\alpha} : \alpha \in \mathcal{A}\}$ of the Euclidean space $\mathbb{R}^n$. The index set $\mathcal{A}$ can be finite or infinite.\footnote{In particular, it is infinite in Proposition~\ref{thm:fixed_util} in Section~\ref{sec:discussion}.} We call $\mathcal{S}$ a \textit{knowledge partition} and any element $S \in \mathcal{S}$ a \textit{knowledge element}. Under Assumption~\ref{ass:author2} in the new context, the author knows which knowledge element $S \in \mathcal{S}$ contains the ground truth $\bR$.

The author is required to pick a knowledge element, say $S$, from the knowledge partition $\mathcal{S}$ and send the message ``the ground truth is in the set $S$'' to the conference. She is \textit{not} allowed to observe the raw scores $\by = \bR + \bz$ while making this decision. Meanwhile, the conference knows nothing about the ground truth $\bR$ but can observe $\by$. Given the message ``$\bR$ is in $S$'' from the author, the conference solves the following optimization program:
\begin{equation}\label{eq:isotone_shape}
\begin{aligned}
&\min_{\br} ~ \| \by - \br\|^2 \\
&\text{~s.t.} ~~ \br \in S,
\end{aligned}
\end{equation}
and employs its solution as an estimator of the ground truth $\bR$. This program is equivalent to projecting $\by$ onto the knowledge element $S$. We denote the solution to the optimization program \eqref{eq:isotone_shape} by $\widehat\bR^S$.

Formally, each knowledge element is a closed set with nonempty interior, and these elements form a partition of $\mathbb{R}^n$ so that $\cup_{\alpha \in \mathcal{A}} S_{\alpha} = \mathbb{R}^n$, with the interiors of any two distinct elements disjoint:\footnote{The boundaries have Lebesgue measure zero. Thus, although some knowledge elements may overlap at their boundaries, for a generic ground-truth vector $\bR$, there is a unique knowledge element containing it.}
\[
\mathring{S}_{\alpha} \cap \mathring{S}_{\alpha'} = \varnothing
\quad \text{for any}\;\alpha \neq \alpha'.
\]
In addition, we assume that the boundary between any two adjacent knowledge elements is a piecewise-smooth surface.\footnote{We do not consider two sets adjacent if they merely share a corner where three or more sets meet, following the same spirit as the Four-Color Theorem.} A surface is called smooth if, at any point on the surface, it can be locally described by $f(x_1,\dots,x_n) = 0$ for some continuously differentiable function $f$ whose gradient $\nabla f$ is nondegenerate.

Following Definition~\ref{def:truth}, we say that the knowledge partition $\mathcal{S}$ is truthful if
\[
\mathbb{E}\bigl[U(\widehat{\bR}^{S^\star})\bigr] \ge \mathbb{E}\bigl[U(\widehat{\bR}^{S})\bigr]
\]
for all $S^\star, S \in \mathcal{S}$ such that $S^\star$ contains the ground truth $\bR$. Our main result in this section gives a necessary condition for a knowledge partition to be truthful; its proof is deferred to Section~\ref{sec:does-there-exist}.

\begin{theorem}\label{thm:compare}
If the author tells the truth whenever Assumptions~\ref{ass:author2}, \ref{ass:noise}, and~\ref{ass:convex} are satisfied, then the boundary between any two adjacent knowledge elements is piecewise-flat, and each flat surface must be part of a pairwise-comparison hyperplane defined by
\[
x_i - x_j = 0
\quad
\text{for some}\; 1 \le i < j \le n.
\]
\end{theorem}

\begin{remark}
This theorem applies even to the trivial knowledge partition containing $\mathbb{R}^n$ as its single element, which is by definition truthful. In that case, the boundary is empty, so the necessary condition in Theorem~\ref{thm:compare} is automatically satisfied.
\end{remark}

\begin{remark}
This characterization of truthful knowledge partitions is obtained by taking an arbitrary convex utility function. For a specific utility, a truthful partition need not be strictly based on pairwise comparisons. See Proposition~\ref{thm:fixed_util} in Section~\ref{sec:discussion}.
\end{remark}

\begin{remark}
As with Theorem~\ref{thm:main}, Theorem~\ref{thm:compare} remains valid if the squared $\ell_2$ loss in \eqref{eq:isotone_shape} is replaced by the sum of Bregman divergences $\sum_{i=1}^n D_{\phi}(y_i, r_i)$.
\end{remark}

The condition in Theorem~\ref{thm:compare} is equivalent to the following: for any point $\bx$, one can identify the knowledge element containing $\bx$ by performing pairwise comparisons of certain coordinates of $\bx$. For example, consider
\begin{equation}\label{eq:largest_i1}
S_i = \{\bx \in \mathbb{R}^n : x_i \text{ is the largest among } x_1,\dots,x_n\}
\quad
\text{for}\; i=1,\dots,n.
\end{equation}
This collection of sets is clearly based on pairwise comparisons, since $\bx \in S_i$ if and only if $x_i \ge x_j$ for all $j\neq i$.

However, the converse of Theorem~\ref{thm:compare} does not hold. Indeed, we will show in Section~\ref{sec:exampl-truth-tell} that some pairwise-comparison-based partitions are truthful (including the one generated by \eqref{eq:largest_i1}), whereas others are not. Since all pairwise-comparison hyperplanes pass through the origin, Theorem~\ref{thm:compare} immediately implies the following:

\begin{corollary}\label{cor:cone}
If a knowledge partition $\mathcal{S}$ is truthful whenever Assumptions~\ref{ass:author2}, \ref{ass:noise}, and~\ref{ass:convex} are satisfied, then every knowledge element $S \in \mathcal{S}$ is a cone. That is, if $\bx \in S$, then $\lambda \bx \in S$ for all $\lambda \ge 0$.
\end{corollary}

Since at least two items are required for pairwise comparison, it seems necessary to have $n \ge 2$ for a truthful knowledge partition to exist. This intuition is confirmed by the following proposition, whose proof appears in the Appendix. Note that Theorem~\ref{thm:compare} assumes $n \ge 2$, so it does not directly imply this result.

\begin{proposition}\label{prop:n=1}
Under Assumptions~\ref{ass:author2}\footnote{For Proposition~\ref{prop:n=1}, the assumption $n \ge 2$ is waived.}, \ref{ass:noise}, and~\ref{ass:convex}, there is no truthful knowledge partition when $n = 1$, except for the trivial case $\mathcal{S} = \{\mathbb{R}\}$.
\end{proposition}

\subsection{Implications for the Isotonic Mechanism}

A prominent example of a pairwise-comparison-based knowledge partition is the isotonic partition formed by the $n!$ isotonic cones. By construction, this partition is generated by pairwise-comparison hyperplanes. Moreover, any region formed by such hyperplanes must be a union of several isotonic cones. Formally, we say that $\mathcal{S}_1$ is ``coarser'' than $\mathcal{S}_2$ if each element of $\mathcal{S}_1$ is a union of some elements of $\mathcal{S}_2$. Hence, we obtain the following corollary of Theorem~\ref{thm:compare}:

\begin{corollary}
If a knowledge partition $\mathcal{S}$ is truthful whenever Assumptions~\ref{ass:author2}, \ref{ass:noise}, and~\ref{ass:convex} are satisfied, then $\mathcal{S}$ is coarser than the isotonic partition. In particular, the cardinality of $\mathcal{S}$ is at most $n!$.
\end{corollary}

\begin{remark}
Since the union of cones is also a cone, this result implies Corollary~\ref{cor:cone}. However, an individual knowledge element may be nonconvex or even noncontiguous.
\end{remark}

Combining this result with Theorem~\ref{thm:main} shows that the isotonic partition is the most fine-grained among all truthful partitions. In particular, any isotonic cone in that partition cannot be further split by pairwise-comparison hyperplanes. This leads to the next corollary, stated as a theorem for emphasis:

\begin{theorem}\label{thm:best}
Under Assumptions~\ref{ass:author2}, \ref{ass:noise}, and~\ref{ass:convex}, the Isotonic Mechanism elicits the most fine-grained information about the ground truth among all truthful partition-based mechanisms of the form~\eqref{eq:isotone_shape}.
\end{theorem}

A key observation here is that there are reasons to prefer a more fine-grained partition to a coarser one, in order to improve estimation accuracy. The next proposition illustrates how the accuracy can depend on the coarseness of the partition. Its proof is deferred to the Appendix.

\begin{proposition}\label{prop:improve}
Suppose the noise vector $\bz$ in the observation $\by = \bR + \bz$ consists of i.i.d.\ normal random variables $\mathcal{N}(0, \sigma^2)$. Let $S_1$ and $S_2$ be two cones satisfying $S_2 \subset S_1$, and both contain the ground truth $\bR$. Then
\[
\limsup_{\sigma \to 0}\,
\frac{\mathbb{E}\|\widehat{\bR}^{S_2} - \bR\|^2}{\mathbb{E}\|\widehat{\bR}^{S_1} - \bR\|^2} \le 1,
\quad\;\;
\limsup_{\sigma \to \infty}\,
\frac{\mathbb{E}\|\widehat{\bR}^{S_2} - \bR\|^2}{\mathbb{E}\|\widehat{\bR}^{S_1} - \bR\|^2} \le 1.
\]
\end{proposition}

Intuitively, if the author truthfully reports a knowledge element with a strictly smaller feasible region, the conference may achieve a smaller squared $\ell_2$ risk. This aligns with the intuition that a valid but tighter constraint can often yield more accurate estimation.

However, Proposition~\ref{prop:improve} by itself does not imply that the Isotonic Mechanism achieves optimal estimation in every scenario solely by virtue of being the most fine-grained truthful partition. This is because Proposition~\ref{prop:improve} does not cover arbitrary fixed noise levels and only takes limit as $\sigma \to 0$ or $\sigma \to \infty$. Extending it to any noise level is a natural direction for future research, and we conjecture that its conclusions would hold more generally.


%% file: coarse.tex
\section{Other truthful knowledge partitions}
\label{sec:exampl-truth-tell}

While we have identified perhaps the most important truthful knowledge partition, it is appealing to search for other truthful pairwise-comparison-based partitions. From a practical perspective, another motivation is that the author might not know the ground-truth ranking precisely and instead only has partial information about it.

To begin with, we present a counterexample to show that the converse of Theorem~\ref{thm:compare} is not true. Consider $\mathcal{S} = \{S_1, S_2\}$, where $S_1 = \{\bx: x_1 \ge x_2 \ge \cdots \ge x_n\}$, $S_2 = \R^n \setminus S_1$, and $\bR = (n\epsilon, (n-1)\epsilon, \ldots, 2\epsilon, \epsilon) \in S_1$ for some
small $\epsilon > 0$. Note that $S_1$ and $S_2$ are separated by pairwise-comparison hyperplanes. Taking utility $U(x) = x^2$ or $\max\{x, 0\}^2$ and letting the noise terms $z_1, \ldots,
z_n$ be i.i.d.~standard normal random variables, we show in the Appendix that the author would be better off reporting $S_2$ instead of $S_1$, the set that truly contains the
ground truth. Thus, this pairwise-comparison-based knowledge partition is not truthful.

In the remainder of this section, we introduce two useful knowledge partitions and show their truthfulness.

\paragraph{Local ranking.} Other than the isotonic partition, perhaps the simplest nontrivial truthful knowledge partitions are induced by local rankings: first partition $\{1, \ldots, n\}$ into several subsets of sizes, say, $n_1, n_2, \ldots, n_p$ such that $n_1 + \cdots + n_p = n$; then the author is asked to provide a ranking of the $n_q$ papers indexed by each subset for $q = 1, \ldots, p$, but does not make any between-subset comparisons.

This concept is exemplified in the following practical scenario:

\begin{example}
A manager oversees a team of $p$ employees. For $1 \leq q \leq p$, the $q^{\textnormal{th}}$ employee produces $n_q$ items and reports to the manager a ranking of the $n_q$ items based on the employee's perceived value. However, no pairwise comparisons are provided to the manager between items produced by different employees.
\end{example}

Formally, letting $\Si(n)$ be a shorthand for the isotonic partition in $n$ dimensions, we can write the resulting
knowledge partition as
\[
\Si(n_1) \times \Si(n_2) \times \cdots \times \Si(n_p),
\]
which has a cardinality of $n_1!n_2!\cdots n_p!$. Recognizing that the overall utility is additively separable, we readily conclude that this knowledge partition is truthful and the manager will report the ground-truth local ranking for each subset.

\paragraph{Coarse ranking.} Another example is induced by a coarse ranking: given $n_1, n_2, \ldots, n_p$ such that $n_1 + n_2 + \cdots + n_p = n$, the author partitions $\{1, 2, \ldots, n\}$ into $p$ ordered subsets $I_1, I_2, \ldots, I_p$ of sizes $n_1, n_2, \ldots, n_p$, respectively; but she does not reveal any comparisons within each subset at all. The conference wishes that the author would report the ground-truth coarse ranking $(I_1^\star, I_2^\star, \ldots, I^\star_p)$, which satisfies
\begin{equation}\label{eq:ranking_block}
\bR_{I_1^\star} \ge \bR_{I_2^\star} \ge \cdots \ge \bR_{I_p^\star}
\end{equation}
(here we write $\ba \ge \bb$ for two vectors of possibly different lengths if any component of $\ba$ is larger than or equal to any component of $\bb$).

For instance, taking $n_q = 1$ for $q = 1, \ldots, p-1$ and $n_p = n - p + 1$, the author is required to rank only the top $p-1$ items. Another example is to consider $p = 10$ and $n_1 = \cdots
= n_{10} = 0.1n$ (assume $n$ is a multiple of $10$), in which case the author shall identify which items are the top $10\%$, which are the next top $10\%$, and so on.

\begin{figure}[!htp]
  \centering
\begin{tikzpicture}[x=1cm, y=1cm, z=-0.6cm]
    \draw [->] (2.5,0,0) -- (3,0,0) node [right] {$x_1$};
    \draw [dashed] (0,0,0) -- (2.5,0,0);
    \draw [->] (0,2.5,0) -- (0,3,0) node [left] {$x_2$};
    \draw [dashed] (0,0,0) -- (0,2.5,0);
    \draw [->] (0,0,1.8) -- (0,0,2.5) node [left] {$x_3$};
    \draw [dashed] (0,0,0) -- (0,0,1.8);
\filldraw [opacity=.5,black] (2.5,0,0) -- (1.25,0,0.9) -- (2.5/3,2.5/3,0.6) -- (1.25,1.25,0) -- cycle;
\filldraw [opacity=.5,gray] (0,2.5,0) -- (1.25,1.25,0) -- (2.5/3,2.5/3,0.6) -- (0,1.25,0.9) -- cycle;
\filldraw [opacity=.5,lightgray] (0,0,1.8) -- (1.25,0,0.9) -- (2.5/3,2.5/3,0.6) -- (0,1.25,0.9) -- cycle;
\node[] at (2.6,0, 3) {$n_1 = 1, n_2 = 2$};
\end{tikzpicture}
\hspace{1cm}
\begin{tikzpicture}[x=1cm, y=1cm, z=-0.6cm]
    \draw [->] (2.5,0,0) -- (3,0,0) node [right] {$x_1$};
    \draw [dashed] (0,0,0) -- (2.5,0,0);
    \draw [->] (0,2.5,0) -- (0,3,0) node [left] {$x_2$};
    \draw [dashed] (0,0,0) -- (0,2.5,0);
    \draw [->] (0,0,1.8) -- (0,0,2.5) node [left] {$x_3$};
    \draw [dashed] (0,0,0) -- (0,0,1.8);
\filldraw [opacity=.5,black] (2.5,0,0) -- (0,2.5,0) -- (0.7,1,0.6) -- cycle;
\filldraw [opacity=.5,gray] (2.5,0,0) -- (0.7,1,0.6) -- (0,0,1.8) -- cycle;
\filldraw [opacity=.5,lightgray] (0.7,1,0.6) -- (0,2.5,0) -- (0,0,1.8) -- cycle;
\node[] at (2.6,0, 3) {$n_1 = 2, n_2 = 1$};
\end{tikzpicture}  
  \caption{Knowledge partitions induced by coarse rankings in $n = 3$ dimensions. The illustration shows a slice of the partitions restricted to the nonnegative orthant.}
  \label{fig:block}
\end{figure}

Writing $\bI := (I_1, \ldots, I_p)$, we denote by
\[
S_{\bI} := \{\bx: \bx_{I_1} \ge \bx_{I_2} \ge \cdots \ge \bx_{I_p}\}
\] 
the knowledge element indexed by $\bI$. There are in total $\frac{n!}{n_1! \cdots n_p!}$ knowledge elements, which together form a knowledge partition. As is evident, any two adjacent knowledge elements are separated by pairwise-comparison hyperplanes. Figure~\ref{fig:block} illustrates two such knowledge partitions in the case $n = 3$. The coarse ranking $\bI$ of the author's choosing may or may not be correct. Nevertheless, this is what the conference would incorporate into the estimation of the ground truth:
\begin{equation}\label{eq:isotone_incomp}
\begin{aligned}
&\min_{\br} ~  \| \by - \br\|^2 \\
&\text{~s.t.} ~~  \br \in S_{\bI},
\end{aligned}
\end{equation}
which is a convex optimization program since the knowledge element $S_{\bI}$ is convex. We call \eqref{eq:isotone_incomp} a coarse Isotonic Mechanism.

A practical application of this mechanism can be illustrated by the following scenario:
\begin{example}
A worker produces $n_q$ products in grade $q$, for $q = 1, \ldots, p$. Products across different grades exhibit substantially distinct values, yet the worker cannot differentiate between products within the same grade category. The products are subsequently shuffled so that only the worker knows the grade of each individual product.

\end{example}

The next result confirms that this new knowledge partition is truthful. Although it is pairwise-comparison-based, Theorem~\ref{thm:block} does not follow directly from Theorem~\ref{thm:compare}; its proof in Section~\ref{sec:proofs-section-ref} relies on different ideas.

\begin{theorem}\label{thm:block}
Under Assumptions~\ref{ass:author2}, \ref{ass:noise}, and \ref{ass:convex}, the expected overall utility is maximized if the author truthfully reports the coarse ranking that fulfills~\eqref{eq:ranking_block}.

\end{theorem}

\begin{remark}
Taking $n_1 = 1$ and $n_2 = n-1$, Theorem~\ref{thm:block} shows that the collection of knowledge elements taking the form \eqref{eq:largest_i1} is a truthful knowledge partition.
\end{remark}

One can construct further truthful partitions by integrating these two types of partitions. Instead of providing a complete ranking of items in each subset (as in the local ranking setting), one could simply offer a coarse ranking for each subset. It is straightforward to verify that the resulting knowledge partition is also truthful. An intriguing open question is whether there exist other truthful partitions that can be derived from these two prototypes.

Finally, we note that Proposition~\ref{thm:estimate} continues to hold for the coarse Isotonic Mechanism, implying that it achieves a smaller squared $\ell_2$ risk than that obtained by using raw scores. However, we conjecture that the coarse Isotonic Mechanism yields inferior estimation performance compared to the standard Isotonic Mechanism (see Proposition \ref{prop:improve} and the discussion that follows).


%% file: extend.tex
\section{Extensions}
\label{sec:extens-util-funct}

In this section, we show that truthfulness continues to be the optimal strategy for the author in more general settings.

\medskip
\noindent{\bf Robustness to inconsistencies.}
The author might not have complete knowledge of the ground-truth ranking in some scenarios, but is certain that some rankings are more consistent than others. More precisely, consider
two rankings $\pi_1$ and $\pi_2$ such that neither is the ground-truth ranking, but the former can be obtained by swapping two entries of the latter in an upward manner in the sense that
\[
R_{\pi_1(i)} = R_{\pi_2(j)} > R_{\pi_1(j)} = R_{\pi_2(i)}
\]
for some $1 \le i < j \le n$ and $\pi_1(k) = \pi_2(k)$ for all $k \ne i, j$. In general, $\pi_1$ is said to be more consistent than $\pi_2$ if $\pi_1$ can be sequentially
swapped from $\pi_2$ in an upward manner.

If the author must choose between $\pi_1$ and $\pi_2$, she would be better off reporting the more consistent ranking, thereby being truthful in a relative sense. This shows the
robustness of the Isotonic Mechanism against inconsistencies in rankings. A proof of this result is presented in the Appendix.

\begin{proposition}\label{thm:incomplete}
Suppose $\pi_1$ is more consistent than $\pi_2$ with respect to the ground truth $\bR$. Under Assumptions~\ref{ass:noise} and \ref{ass:convex}, reporting $\pi_1$ yields higher or equal
overall utility in expectation under the Isotonic Mechanism than reporting $\pi_2$.

\end{proposition}

Intuitively, one might expect that a more consistent ranking would also lead to better estimation performance. If this intuition were true, it would lead to an extension of Proposition~\ref{thm:estimate}. We leave this interesting question to future research.




\bigskip
\noindent{\bf Multiple knowledge partitions.} Given several truthful knowledge partitions, say, $\mathcal S_1, \ldots, \mathcal S_K$, one can offer the author the freedom of choosing any
knowledge element from these partitions. The resulting mechanism remains truthful. Formally, we have the following result.

\begin{proposition}\label{prop:mix}
Let $\mathcal S_1, \ldots, \mathcal S_K$ be truthful knowledge partitions. If the author is required to report one knowledge element from any of these knowledge partitions, then she must be
truthful in order to maximize her expected overall utility as much as possible.

\end{proposition}

That is, if the author chooses some knowledge element $S \in \mathcal S_k$ for $1 \le k \le K$ such that $S$ does not contain the ground truth $\bR$, she can
always improve her overall utility in expectation by reporting the knowledge element in $\mathcal S_k$ truly containing $\bR$. She can randomly pick a truthful knowledge element from
any of $\mathcal S_1, \ldots, \mathcal S_K$ when it is unclear which knowledge partition leads to the highest overall utility. As Proposition~\ref{thm:estimate} still holds in
the case of multiple knowledge partitions, honesty would always lead to better estimation accuracy than using the raw observation as long as all knowledge elements are convex.

This result allows for certain flexibility in truthfully eliciting information, especially when we are not sure which knowledge partition satisfies Assumption~\ref{ass:author2}. An immediate application is to take several knowledge partitions induced by coarse rankings~\eqref{eq:ranking_block} in the hope that, for at least one knowledge partition, the author can determine
the truthful knowledge element. For example, it seems plausible to take approximately equal sizes for the subsets: 
\[
n_1 \approx n_2 \approx \cdots \approx n_p \approx \frac{n}{p}.
\] 
However, the author might not have sufficient knowledge about her submissions to provide the true coarse ranking, thereby violating Assumption~\ref{ass:author2}. To circumvent this issue, we can let the author pick any coarse ranking such that the number of subsets $p$ is not smaller than, say, $\sqrt{n}$, and the largest subset size $\max_{1 \le i \le p} n_i$ is not greater than, say, $n/10$. 


\bigskip
\noindent{\bf Nonseparable utility functions.} The overall utility in Assumption~\ref{ass:convex} can be generalized to certain nonseparable functions. Explicitly, let the overall utility
function $U(\bx)$ be symmetric in its $n$ coordinates and satisfy
\begin{equation}\label{eq:shcurx}
(x_i - x_j) \left(\frac{\partial U(\bx)}{\partial x_i}  - \frac{\partial U(\bx)}{\partial x_j}  \right) \ge 0
\end{equation}
for all $\bx$. The following result shows that the author's optimal strategy continues to be truthful.

\begin{proposition}\label{prop:schur}
Under Assumptions~\ref{ass:author2} and \ref{ass:noise}, the Isotonic Mechanism is truthful if the author aims to maximize this type of overall utility in expectation.

\end{proposition}

Theorem~\ref{thm:main} is implied by this proposition when the (univariate) utility function $U$ is differentiable. To see this point, note that a separable overall utility $U(\bx) =
U(x_1) + \cdots + U(x_n)$ in
Assumption~\ref{ass:convex} satisfies
\[
\begin{aligned}
(x_i - x_j) \left(\frac{\partial U(\bx)}{\partial x_i}  - \frac{\partial U(\bx)}{\partial x_j}  \right) = (x_i - x_j) \left( U'(x_i)  - U'(x_j)\right).
\end{aligned}
\]
Since $U'$ is a nondecreasing function, we get $(x_i - x_j) \left( U'(x_i)  - U'(x_j)\right) \ge 0$.

On the other hand, the applicability of Proposition~\ref{prop:schur} is broader than that of Theorem~\ref{thm:main} as there are symmetric functions that satisfy \eqref{eq:shcurx} but are
not separable. A simple example is $U(\bx) = \max\{x_1, x_2, \ldots, x_n\}$, and an author with this overall utility is only concerned with the highest-scored paper.\footnote{This overall
  utility obeys \eqref{eq:shcurx}, but partial derivatives shall be replaced by partial subdifferentials. This is also the case for the next example.} More generally, letting $x_{(1)}
\ge x_{(2)} \ge \cdots \ge x_{(k)}$ be the $k \le n$ largest entries of $\bx$, this proposition also applies to
\[
U(\bx) = h(x_{(1)}) + h(x_{(2)}) + \cdots + h(x_{(k)})
\]
for any nondecreasing convex function $h$.

Proposition~\ref{prop:schur} follows from the proof of Theorem~\ref{thm:main} in conjunction with Remark~\ref{rm:schur} in Section~\ref{sec:prop:score}.

\bigskip
\noindent{\bf Multiple authors.} Papers in machine learning are frequently authored by multiple researchers. In this context, we introduce a variant of the Isotonic Mechanism that addresses collaborative authorship scenarios.

To formalize the problem, consider a setting with $n$ papers and $M$ authors. Let $\ind_{ij} = 1$ if the $i^{\textnormal{th}}$ paper is co-authored by the $j^{\textnormal{th}}$ author for $1 \le i \le n$ and $1 \le j \le M$, and $\ind_{ij} = 0$ otherwise. This yields a bipartite graph between authors and submissions that can be encoded by an $n \times M$ binary matrix. Using this authorship matrix as input, Algorithm~\ref{alg:dpsgd1} partitions the $n$ papers into disjoint groups such that papers within each group share the same author, while different groups correspond to distinct authors. The Isotonic Mechanism is then applied to each group independently. Due to the non-overlapping nature of these groups, the Isotonic Mechanism maintains truthfulness across all groups in the partition.

\begin{algorithm}[!htb]
    \caption{Partition of the bipartite graph for the Isotonic Mechanism}\label{alg:dpsgd1}
    \begin{algorithmic}[0]
        \State{\bf Input:} Set of papers $\I$, set of authors $\O$, and the authorship matrix $\{\ind_{ij}\}_{1 \le i \le n, 1 \le j \le M}$
        \While{$\I$ is \textit{not} empty}
        \For{Author $o \in \O$}
        \Statex \hspace{1.1cm} Let $w_o$ be the number of papers in $\I$ shared by $o$        
        \EndFor
        \State Find the author $o^\star$ such that $o^\star = \argmax_{o \in \O} w_o$ 
                \Comment{\textit{Randomly choose $o^\star$ when ties occur}}
        \State Apply the Isotonic Mechanism to $o^\star$ and the papers in $\I$ shared by $o^\star$
        \State Update the sets: $O \leftarrow O\setminus \{o^\star\}$ and $\I \leftarrow \I\setminus \{\text{papers shared by } o^\star\}$
        \EndWhile
    \end{algorithmic}
\end{algorithm}

In this algorithm, one reasonable criterion is to prioritize partitions containing larger groups, which effectively gives precedence to authors with more paper submissions. However, this approach could result in the formation of singleton groups. A practical limitation of this algorithm is that authors with many submissions might possess less detailed familiarity with each paper, potentially yielding less reliable rankings compared to those from authors with fewer submissions. To mitigate this issue, one potential solution is to designate one or two authors per paper who are best qualified to assess it. This designation could be accomplished through anonymous peer assessments of co-authors based on their familiarity with the paper's content. While numerous alternative approaches exist to address this challenge, we defer their comprehensive investigation to future research.

%% file: proofs.tex
\section{Proofs}
\label{sec:proofs}

Here, we prove Theorems \ref{thm:main}, \ref{prop:score_dependent}, \ref{thm:compare}, and \ref{thm:block}. Proofs of other technical results in the paper are relegated to the Appendix.

\subsection{Proof of Theorem~\ref{thm:main}}
\label{sec:xxxxxxxxxxxx}


The following definition and lemma will be used in the proof of this theorem.

\begin{definition}
We say that a vector $\ba \in \R^n$ \textit{majorizes} $\bb \in \R^n$ in the \textit{natural order}, denoted $\ba \succeqn b$, if 
\begin{equation}\nonumber
\sum_{i=1}^k a_{i} \ge \sum_{i=1}^k b_{i}
\end{equation}
for all $1 \le k \le n$, with equality when $k = n$. 
\end{definition}

A departure of this definition from weak majorization or majorization is that majorization in the natural order is not invariant under permutations.


In the lemma below, we write $\ba^+$ as a shorthand for the projection of $\ba$ onto the standard isotonic cone $\{\bx: x_1 \ge x_2 \ge \cdots \ge x_n\}$. A proof of this lemma is given in Section~\ref{sec:proof-lemmaxxxx}.

\begin{lemma}\label{lm:maj}
If $\ba \succeqn \bb$, then we have $\ba^+ \succeqn \bb^+$.

\end{lemma}

\begin{remark}
Both $\ba^+$ and $\bb^+$ have already been ordered from the largest to the smallest. Hence, majorization in the natural order ($\succeqn$) is identical to majorization ($\succeq$) in the
case of $\ba^+$ and $\bb^+$.
\end{remark}

Now we are in a position to prove Theorem~\ref{thm:main}. Write $\pi \circ \bm a := (a_{\pi(1)}, a_{\pi(2)}, \ldots, a_{\pi(n)})$ for $\bm a$ under permutation $\pi$.
\proof{Proof of Theorem~\ref{thm:main}}
Assume without loss of generality that $R_1 \ge R_2 \ge \cdots \ge R_n$. In this case, the ground-truth ranking $\pi^\star$ is the identity, that is, $\pi^\star(i) = i$ for all $i$, and the optimization program \eqref{eq:isotone} for the Isotonic Mechanism is
\begin{equation}\nonumber
\begin{aligned}
&\min~ \| \by - \br\|^2 \\
&\text{~s.t.} ~~ r_1 \ge r_2 \ge \cdots \ge r_n.
\end{aligned}
\end{equation}
Its solution is the projection of $\by$ onto the isotonic cone $\{\bx: x_1 \ge x_2 \ge \cdots \ge x_n\}$, that is, $\by^+ = (\bR + \bz)^+$. 

Consider the optimization program with a different ranking $\pi$,
\begin{equation}\label{eq:isot_pi}
\begin{aligned}
&\min ~  \| \by - \br\|^2 \\
&\text{~s.t.}  ~~  r_{\pi(1)} \ge r_{\pi(2)} \ge \cdots \ge r_{\pi(n)}.
\end{aligned}
\end{equation}
This is equivalent to
\begin{equation}\nonumber
\begin{aligned}
&\min ~   \| \pi \circ \by - \tilde\br\|^2 \\
&\text{~s.t.} ~~ \tilde r_1 \ge \tilde r_{2} \ge \cdots \ge \tilde r_n,
\end{aligned}
\end{equation}
with the relationship $\tilde\br = \pi\circ \br$. From this equivalence it is easy to see that the solution to \eqref{eq:isot_pi} can be written as $\pi^{-1} \circ (\pi \circ \by)^+ =
\pi^{-1} \circ (\pi \circ \bR + \pi \circ \bz)^+$. It suffices to show that the overall utility obeys
\[
\E U\left( (\bR + \bz)^+\right) \ge \E U \left(\pi^{-1} \circ (\pi \circ \bR + \pi \circ \bz)^+\right) = \E U \left( (\pi \circ \bR + \pi \circ \bz)^+\right),
\]
where the equality follows because the overall utility is invariant under permutations. Under Assumption~\ref{ass:noise}, the entries $z_1, \ldots, z_n$ of $\bz$ are exchangeable random
variables. This gives
\[
\E U \left( (\pi \circ \bR + \pi \circ \bz)^+\right) = \E U \left( (\pi \circ \bR + \bz)^+\right).
\]
Thus, the proof is complete if we prove
\begin{equation}\label{eq:thm2_key}
\E U\left( (\bR + \bz)^+\right) \ge \E U \left( (\pi \circ \bR + \bz)^+\right).
\end{equation}

To prove \eqref{eq:thm2_key}, we utilize the following crucial fact
\[
\bR + \bz \succeqn \pi \circ\bR + \bz.
\]
This holds because $R_1, \ldots, R_n$ are already in descending order. Therefore, it merely follows from Lemma~\ref{lm:maj} that
\[
(\bR + \bz)^+ \succeqn (\pi \circ\bR + \bz)^+
\]
or, equivalently,
\[
(\bR + \bz)^+ \succeq (\pi \circ\bR + \bz)^+.
\]
By Lemma~\ref{lm:jensen}, we get
\[
\sum_{i=1}^n U \left( (\bR + \bz)^+_i\right) \ge \sum_{i=1}^n U \left( (\pi \circ \bR + \bz)^+_i\right)
\]
for any convex function $U$, which implies \eqref{eq:thm2_key}. This completes the proof.

\endproof

\subsubsection{Proof of Lemma~\ref{lm:maj}}
\label{sec:proof-lemmaxxxx}

\begin{definition}\label{def:transport}
We say that $\bzz^1$ is an upward transport of $\bzz^2$ if there exists $1 \le i < j \le n$ such that $\zz^1_k = \zz^2_k$ for all $k \ne i, j$, $\zz^1_i + \zz^1_j = \zz^2_i + \zz^2_j$, and $\zz^1_i \ge \zz^2_i$.
\end{definition}

Equivalently, $\bzz^1$ is an upward transport of $\bzz^2$ if $\bzz^1$ can be obtained by moving some ``mass'' from an entry of $\bzz^2$ to an earlier entry. As is evident, we have $\bzz^1 \succeqn \bzz^2$ if $\bzz^1$ is an upward transport of $\bzz^2$.

The following lemmas state two useful properties of this relationship between two vectors.

\begin{lemma}\label{lm:seq}
Let $\ba \succeqn \bb$. Then there exists a positive integer $L$ and $\bzz^1, \ldots, \bzz^L$ such that $\bzz^1 = \ba, \bzz^L = \bb$, and $\bzz^l$ is an upward transport of $\bzz^{l+1}$ for
$1 \le l  \le L-1$.

\end{lemma}

Next, recall that $\ba^+$ denotes the projection of $\ba$ onto the standard isotonic cone $\{\bx: x_1 \ge x_2 \ge \cdots \ge x_n\}$.
\begin{lemma}\label{lm:swap}
Let $\ba$ be an upward transport of $\bb$. Then, we have $\ba^+ \succeqn \bb^+$.

\end{lemma}

Lemma~\ref{lm:maj} readily follows from Lemmas~\ref{lm:seq} and \ref{lm:swap}. To see this point, for any $\ba, \bb$ satisfying $\ba \succeqn \bb$, note that by Lemma~\ref{lm:seq} we can find $\bzz^1 = \ba, \bzz^2, \ldots, \bzz^{L-1}, \bzz^L= \bb$ such that $\bzz^l$ is an upward transport of $\bzz^{l+1}$ for $l = 1, \ldots, L-1$. Next, Lemma~\ref{lm:swap} asserts that 
\[
(\bzz^l)^+ \succeqn (\bzz^{l+1})^+
\]
for all $l = 1, \ldots, L-1$. Owing to the transitivity of majorization in the natural order, we conclude that $\ba \succeqn \bb$, thereby proving Lemma~\ref{lm:maj}.

Below, we prove Lemmas~\ref{lm:seq} and \ref{lm:swap}.
\proof{Proof of Lemma~\ref{lm:seq}}
We prove by induction. The base case $n =1$ is clearly true. Suppose this lemma is true for $n$. 

Now we prove the lemma for the case $n + 1$. Let $\bzz^1 = \ba = (a_1, a_2, \ldots, a_{n+1})$ and $\bzz^2 := (b_1, a_1 + a_2 - b_1, a_3, a_4, \ldots, a_{n+1})$. As is evident, $\bzz^1$ is an upward transport of $\bzz^2$. 

Now we consider operations on the components except for the first one. Let $\ba' := (a_1 + a_2 - b_1, a_3, a_4, \ldots, a_{n+1})$ and $\bb' := (b_2, \ldots, b_{n+1})$ be derived by removing the first component of $\bzz^2$ and $\bb$, respectively. These two vectors obey $\ba' \succeqn \bb'$. To see this, note that $a'_1 = a_1 + a_2 - b_1 \ge b_1 + b_2 - b_1 = b_2 = b'_1$, and 
\[
a'_1 + \cdots + a'_k = (a_1 + a_2 - b_1) + a_3 + \cdots + a_{k+1} = \sum_{i=1}^{k+1} a_i - b_1 \ge \sum_{i=1}^{k+1} b_i - b_1 = \sum_{i=2}^{k+1} b_i = b'_1 + \cdots + b'_{k}
\]
for $2 \le k \le n-1 $. Moreover, it also holds that
\[
a'_1 + \cdots + a'_n = (a_1 + a_2 - b_1) + a_3 + \cdots + a_{n+1} = \sum_{i=1}^{n+1} a_i - b_1 = \sum_{i=1}^{n+1} b_i - b_1 = b'_1 + \cdots + b'_{n}.
\]

Thus, by induction, there must exist $\bzz^{'1}, \ldots, \bzz^{'L}$ such that $\bzz^{'1} = \ba', \bzz^{'L} = \bb'$, and $\bzz^{'l}$ is an upward transport of $\bzz^{'l+1}$ for $l = 1, \ldots, L-1$. We finish the proof for $n+1$ by recognizing that $\bzz^1 \equiv \ba, (b_1, \bzz^{'1}), (b_1, \bzz^{'2}), \ldots, (b_1, \bzz^{'L}) \equiv \bb$ satisfy the conclusion of this lemma.

\endproof


To prove Lemma \ref{lm:swap}, we need the following two lemmas. We relegate the proofs of these two lemmas to the Appendix. Denote by $\bm{\mathrm{e}}_i$ the $i^{\textnormal{th}}$ canonical-basis vector in $\R^n$.

\begin{lemma}\label{lm:mono}
For any $\delta > 0$ and $i = 1, \ldots, n$, we have $(\ba + \delta \bm{\mathrm{e}}_i)^+ \ge \ba^+$ in the component-wise sense.
\end{lemma}

\begin{remark}\label{rm:64last}
Likewise, the proof of Lemma~\ref{lm:mono} reveals that $(\ba - \delta \bm{\mathrm{e}}_i)^+ \le \ba^+$. As an aside, recognizing the mean-preserving constraint of isotonic regression, we have $\bm{1}^\top (\ba + \delta \bm{\mathrm{e}}_i)^+ = \bm{1}^\top \ba^+ + \delta$, where $\bm{1} \in \R^n$ denotes the ones vector.
\end{remark}

\begin{lemma}\label{lm:single_val}
Denote by $\bar a$ the sample mean of $\ba$. Then $\ba^+$ has constant entries---that is, $a^+_1 = \cdots = a^+_n$---if and only if
\[
\frac{a_1 + \cdots + a_k}{k} \le \bar a
\]
for all $k =1, \ldots, n$.

\end{lemma}


\proof{Proof of Lemma~\ref{lm:swap}}

Let $1 \le i < j \le n$ be the indices such that $a_i + a_j = b_i + b_j$ and $a_i \ge b_i$. Write $\delta := a_i - b_i \ge 0$. Then, $\bb = \ba - \delta \bm{\mathrm{e}}_i + \delta \bm{\mathrm{e}}_j$. If $\delta = 0$, then $\ba^+ = \bb^+$ because $\ba = \bb$, in which case the lemma holds trivially. In the remainder of the proof, we focus on the nontrivial case $\delta > 0$. 

The lemma amounts to saying that $\ba^+ \succeqn (\ba - \delta \bm{\mathrm{e}}_i + \delta \bm{\mathrm{e}}_j)^+$ for all $\delta > 0$. Owing to the continuity of the projection, it is sufficient to prove the following statement: there exists $\delta_0 > 0$ (depending on $\ba$) such that $\ba^+ \succeqn (\ba - \delta \bm{\mathrm{e}}_i + \delta \bm{\mathrm{e}}_j)^+$.

Let $I$ be the set of indices where the entries of $\ba^+$ has the same value as $i$: $I = \{k: a^+_k = a^+_i\}$. Likewise, define $J = \{k: a^+_k = a^+_j\}$.
There are exactly two cases, namely, $I = J$ and $I \cap J = \emptyset$, which we discuss in the sequel.

\paragraph{Case 1:} $I = J$. For convenience, write $I = \{i_1, i_1+1, \ldots, i_2\}$. By Lemma~\ref{lm:single_val}, we have
\[
\frac{a_{i_1} + a_{i_1+1} + \ldots + a_{i_1+l-1}}{l} \le \bar a_I := \frac{a_{i_1} + x_{i_1+1} + \ldots + x_{i_2}}{i_2-i_1+1}
\]
for $l = 1, \ldots, i_2-i_1+1$.

Now we consider $\bb = \ba - \delta \bm{\mathrm{e}}_i + \delta \bm{\mathrm{e}}_j$ restricted to $I$. Assume that $\delta$ is sufficiently small so that the constant pieces of $\bb^+$ before and after $I$ are the same as those of $\ba^+$. Since $i_1 \le i < j \le i_2$, we have
\[
b_{i_1} + b_{{i_1}+1} + \ldots + b_{i_2} = a_{i_1} + a_{{i_1}+1} + \ldots + a_{i_2}.
\]
On the other hand, we have
\[
b_{i_1} + b_{{i_1}+1} + \ldots + b_{{i_1}+l-1} \le a_{i_1} + a_{{i_1}+1} + \ldots + a_{{i_1}+l-1}
\]
since the index $i$ comes earlier than $j$. Taken together, these observations give
\[
\frac{b_{i_1} + b_{{i_1}+1} + \ldots + b_{{i_1}+l-1}}{l} \le \frac{b_{i_1} + b_{{i_1}+1} + \ldots + b_{i_2}}{i_2-{i_1}+1}
\]
 for all $l = 1, \ldots, i_2-i_1+1$. It follows from Lemma~\ref{lm:single_val} that the projection $\bb^+ = (\ba - \delta \bm{\mathrm{e}}_i + \delta \bm{\mathrm{e}}_j)^+$ remains constant on the set $I$ and this value is the same as $\ba^+$ on $I$ since $b_{i_1} + b_{{i_1}+1} + \ldots + b_{i_2} = a_{i_1} + a_{{i_1}+1} + \ldots + a_{i_2}$. That is, we have $\bb^+ = \ba^+$ in this case.

\paragraph{Case 2:} $I \cap J = \emptyset$. As earlier, let $\delta$ be sufficiently small. Write $I = \{i_1, i_1+1, \ldots, i_2\}$ and $J = \{j_1, j_1+1, \ldots, j_2\}$, where $i_2 < j_1$. Since the isotonic constraint is inactive between the $(i_1-1)^{\textnormal{th}}$ and $i_1^{\textnormal{th}}$ components, the projection $\ba^+_I$ restricted to $I$ is the same as projecting $\ba_I$ onto the $|I| = (i_2-i_1+1)$-dimensional standard isotonic cone. As $\delta$ is sufficiently small, the projection $(\ba - \delta \bm{\mathrm{e}}_i + \delta \bm{\mathrm{e}}_j)^+_I$ restricted to $I$ is also the same as projecting $(\ba - \delta \bm{\mathrm{e}}_i + \delta \bm{\mathrm{e}}_j)_I$ onto the $|I| = (i_2-i_1+1)$-dimensional standard isotonic cone. 

However, since $i \in I$ but $j \notin J$, we see that $(\ba - \delta \bm{\mathrm{e}}_i + \delta \bm{\mathrm{e}}_j)_I = \ba_I - \delta \bm{\mathrm{e}}_i$, where $\bm{\mathrm{e}}_i$ now should be regarded as the $(i-i_1+1)^{\textnormal{th}}$ canonical-basis vector in the reduced $(i_2-i_1+1)$-dimensional space. Then, by Lemma~\ref{lm:mono} and Remark~\ref{rm:64last}, we see that
\[
\bb^+_I = (\ba_I - \delta \bm{\mathrm{e}}_i)^+ \le \ba_I^+
\]
in the component-wise sense, which, together with the fact that $b^+_l = a^+_l$ for $l \in  \{1, \ldots, i_1-1 \} \cup \{i_2+1, \ldots, j_1-1\} \cup \{j_2+1, \ldots, n\}$, gives
\[
b^+_1 + \cdots + b^+_l \le a^+_1 + \cdots + a^+_l
\]
for all $l = 1, \ldots, j_1-1$. Moreover, 
\begin{equation}\label{eq:I_diff}
\begin{aligned}
b^+_1 + \cdots + b^+_l - (a^+_1 + \cdots + a^+_l) &=  b^+_{i_1} + \cdots + b^+_{i_2} - (a^+_{i_1} + \cdots + a^+_{i_2}) \\
&=  b_{i_1} + \cdots + b_{i_2} - (a_{i_1} + \cdots + a_{i_2}) \\
& = - \delta
\end{aligned}
\end{equation}
when $i_2+1 \le l \le j_1-1$.

Now we turn to the case $j_1 \le l \le j_2$. As earlier, for sufficiently small $\delta$, the projection $(\ba - \delta \bm{\mathrm{e}}_i + \delta \bm{\mathrm{e}}_j)^+_J$ restricted to $J$ is the same as projecting $(\ba - \delta \bm{\mathrm{e}}_i + \delta \bm{\mathrm{e}}_j)_J$ onto the $|J| = (j_2-j_1+1)$-dimensional standard isotonic cone. Then, since $\bb_J = (\ba - \delta \bm{\mathrm{e}}_i + \delta \bm{\mathrm{e}}_j)_J = \ba_J + \delta \bm{\mathrm{e}}_j$, it follows from Lemma~\ref{lm:mono} that 
\begin{equation}\label{eq:J_compo}
\bb^+_J \ge \ba^+_J,
\end{equation}
and meanwhile, we have
\begin{equation}\label{eq:J_diff}
\begin{aligned}
b^+_{j_1} + \cdots + b^+_{j_2} - (a^+_{j_1} + \cdots + a^+_{j_2}) =  b_{j_1} + \cdots + b_{j_2} - (a_{j_1} + \cdots + a_{j_2}) =  \delta.
\end{aligned}
\end{equation}
Thus, for any $j_1 \le l \le j_2$, \eqref{eq:J_compo} and \eqref{eq:J_diff} give
\begin{equation}\nonumber
\begin{aligned}
b^+_{j_1} + \cdots + b^+_l - (a^+_{j_1} + \cdots + a^+_l) \le b^+_{j_1} + \cdots + b^+_{j_2} - (a^+_{j_1} + \cdots + a^+_{j_2}) = \delta.
\end{aligned}
\end{equation}
Therefore, we get
\[
\begin{aligned}
&b^+_1 + \cdots + b^+_l - (a^+_1 + \cdots + a^+_l) \\
&= b^+_1 + \cdots + b^+_{j_1-1} - (a^+_1 + \cdots + a^+_{j_1-1}) + b^+_{j_1} + \cdots + b^+_l - (a^+_{j_1} + \cdots + a^+_l)\\
&= -\delta + b^+_{j_1} + \cdots + b^+_l - (a^+_{j_1} + \cdots + a^+_l)\\
& \le -\delta + \delta\\
& = 0,
\end{aligned}
\]
where the second equality follows from \eqref{eq:I_diff}.

Taken together, the results above show that
\[
b^+_1 + \cdots + b^+_l \le a^+_1 + \cdots + a^+_l
\]
for $1 \le l \le j_2$, with equality when $l \le i_1-1$ or $l =j_2$. In addition, this inequality remains true---in fact, reduced to equality---when $l > j_2$. This completes the proof.

\endproof

\subsection{Proof of Theorem~\ref{prop:score_dependent}}
\label{sec:prop:score}

Define
\begin{equation}\label{eq:tildeu}
\widetilde U(\bx) = \sum_{i=1}^n U(x_i; R_{\rho(i)}),
\end{equation}
where $\rho$ is a permutation such that $\bx$ and $\bR_{\rho}$ have the same descending order. For example, if $x_l$ is the largest element of $\bx$, so is $R_{\rho(l)}$ the
largest element of $\bR_{\rho}$. By construction, $\widetilde U$ is symmetric. Moreover, this function satisfies the following two
lemmas. The proofs are given later in this subsection.

\begin{lemma}\label{lm:u_perm}
Under Assumption~\ref{ass:convex2}, the overall utility satisfies
\[
\sum_{i=1}^n U(x_i; R_i) \le \widetilde U(\bx).
\]

\end{lemma}

\begin{lemma}\label{lm:scc}
Under Assumption~\ref{ass:convex2}, the function $\widetilde U$ is Schur-convex in the sense that $\widetilde U(\ba) \ge \widetilde U(\bb)$ is implied by $\ba \succeq \bb$.

\end{lemma}

Now we are ready to prove Theorem~\ref{prop:score_dependent}.
\proof{Proof of Theorem~\ref{prop:score_dependent}}

Assume without loss of generality that $R_1 \ge R_2 \ge \cdots \ge R_n$. As earlier, denote by $\widehat \bR^{\pi}$ the ranking-calibrated scores produced by the Isotonic Mechanism provided ranking $\pi$. For simplicity, write
$\widehat \bR = \widehat \bR^{\pi^\star}$ when the ranking is the ground-truth ranking $\pi^\star$. Note that $\pi^\star(i) = i$ for all $i$, and $\widehat\bR$ and $\bR$ have the same descending
order. As such, we get
\[
\widetilde U(\widehat\bR) = \sum_{i=1}^n U( \widehat R_i; R_i).
\]

To prove
\[
\E \left[ \sum_{i=1}^n U( \widehat R_i; R_i) \right] \ge \E \left[ \sum_{i=1}^n U( \widehat R^{\pi}_{i}; R_i) \right],
\]
we start by observing that
\[
\widetilde U(\widehat\bR^{\pi}) \ge \sum_{i=1}^n U( \widehat R^{\pi}_{i}; R_i)
\]
is an immediate consequence of Lemma~\ref{lm:u_perm}. Hence, it is sufficient to prove
\begin{equation}\label{eq:thm3_ine}
\E \widetilde U(\widehat\bR) \ge \E \widetilde U(\widehat\bR^{\pi}).
\end{equation}

As in the proof of Theorem~\ref{thm:main}, it follows from Lemma~\ref{lm:maj} that 
\[
\widehat\bR = (\bR + \bz)^+ \succeq (\pi \circ \bR + \bz)^+.
\]
As Lemma~\ref{lm:scc} ensures that $\widetilde U$ is Schur-convex, the majorization relation above gives
\begin{equation}\label{eq:thm3_z}
\widetilde U(\widehat\bR) \ge \widetilde U((\pi \circ \bR + \bz)^+).
\end{equation}
Moreover, the coupling argument in the proof of Theorem~\ref{thm:main} implies that $(\pi \circ \bR + \bz)^+$ has the same probability distribution as $\widehat\bR^{\pi}$, which gives
\[
\E \widetilde U((\pi \circ \bR + \bz)^+) = \E \widetilde U(\widehat\bR^{\pi}).
\]
Together with \eqref{eq:thm3_z}, this equality implies \eqref{eq:thm3_ine}.

\endproof

Next, we turn to the proof of Lemma~\ref{lm:u_perm}.
\proof{Proof of Lemma~\ref{lm:u_perm}}

Given two permutations $\pi_1$ and $\pi_2$, if there exist two indices $i,j$ such that $\pi_1(k) = \pi_2(k)$ for all $k \ne i,j$ and $R_{\pi_1(i)} - R_{\pi_1(j)} = -(R_{\pi_2(i)} - R_{\pi_2(j)})$ has the same sign as $x_i - x_j$, we say that $\pi_1$ is an upward swap of $\pi_2$ with respect to $\bx$. As is evident, the permutation $\rho$ in \eqref{eq:tildeu} can be obtained by sequentially swapping the identity permutation in an upward manner with respect to $\bx$. Therefore, it suffices to prove the lemma in the case $n =2$. Specifically, we only need to prove that
\begin{equation}\label{eq:u_perm1}
U(x_1; R_1) + U(x_2; R_2)  \le U(x_1; R_2) + U(x_2; R_1)
\end{equation}
if $x_1 \ge x_2$ and $R_1 \le R_2$. 

Define
\[
g(x) = U(x; R_2) - U(x_2; R_2) - U(x; R_1) + U(x_2; R_1).
\]
Then, \eqref{eq:u_perm1} is equivalent to $g(x) \ge 0$ for $x \ge x_2$. To prove this, observe that
\[
g'(x) =  \frac{\d U(x; R_2)}{\d x} - \frac{\d U(x; R_1)}{\d x} \ge 0
\]
by Assumption~\ref{ass:convex2}. This establishes \eqref{eq:u_perm1}, thereby completing the proof.

\endproof

Next, we turn to the proof of Lemma~\ref{lm:scc}, for which we need the following lemma. For a proof of this lemma, see~\citet{marshall1979inequalities}.  

\begin{lemma}[Schur--Ostrowski criterion]
If a function $f: \R^n \goto \R $ is differentiable. Then $f$ is Schur-convex if and only if it is symmetric and satisfies
\[
(x_i - x_j)\left(\frac{\partial f}{\partial x_i} - \frac{\partial f}{\partial x_j}\right) \ge 0 
\]
for all $1 \le i \ne j \le n$.

\end{lemma}

\begin{remark}\label{rm:schur}
The condition on the overall utility in Proposition~\ref{prop:schur} is precisely Schur-convexity. Thus, Proposition~\ref{prop:schur} follows from the proof of Theorem~\ref{prop:score_dependent}.
\end{remark}

\proof{Proof of Lemma~\ref{lm:scc}}
First, consider the case where all elements of $\bx$ are different from each other. Without loss of generality, assume $x_i > x_j$. It suffices to prove that
\begin{equation}\label{eq:u_inc}
\frac{\partial \widetilde U(\bx)}{\partial x_i} - \frac{\partial \widetilde U(\bx)}{\partial x_j} = \frac{\d U(x; R_{\rho(i)})}{\d x}\Big|_{x = x_i} - \frac{\d U(x; R_{\rho(j)})}{\d
  x}\Big|_{x = x_j} \ge 0.
\end{equation}

Since $U(x; R_{\rho(i)})$ is a convex function in $x$, we have
\[
\frac{\d U(x; R_{\rho(i)})}{\d x}\Big|_{x = x_i} - \frac{\d U(x; R_{\rho(i)})}{\d x}\Big|_{x = x_j} \ge 0
\]
as the derivative of a convex function is a nondecreasing function. Next, recognizing that $R_{\rho(i)} \ge R_{\rho(j)}$ is implied by the construction of the permutation $\rho$, it follows from Assumption~\ref{ass:convex2} that
\[
\frac{\d U(x; R_{\rho(i)})}{\d x}\Big|_{x = x_j} - \frac{\d U(x; R_{\rho(j)})}{\d x}\Big|_{x = x_j} \ge 0.
\]
Adding the last two inequalities, we arrive at \eqref{eq:u_inc}.

If $\bx$ has ties---for example, $x_i = x_{i'}$ for some $i' \ne i$---then $\widetilde U$ is one-sided differentiable with respect to $x_i$ at $\bx$. Indeed, the right derivative 
\[
\frac{\partial_+ \widetilde U(\bx)}{\partial x_i} = \frac{\d U(x; \max\{R_{\rho(i)}, R_{\rho(i')}\})}{\d x}\Big|_{x = x_i},
\]
while the left derivative
\[
\frac{\partial_- \widetilde U(\bx)}{\partial x_i} = \frac{\d U(x; \min\{R_{\rho(i)}, R_{\rho(i')}\})}{\d x}\Big|_{x = x_i}.
\]
Other than this difference, the remainder resembles the proof in the earlier case. For example, we still have $R_{\rho(j)} \le \min\{R_{\rho(i)}, R_{\rho(i')}\}$ and $R_{\rho(j')} \le
\min\{R_{\rho(i)}, R_{\rho(i')}\}$ for any $j'$ such that $x_j = x_{j'}$. Thus, details are omitted.

\endproof

\subsection{Proof of Theorem~\ref{thm:compare}}
\label{sec:does-there-exist}

We prove this theorem in a slightly more general setting where \eqref{eq:isotone_shape} is replaced by \eqref{eq:isotone_shape_d}. That is, the squared error loss is replaced by the sum
of Bregman divergences. We start by introducing the following definition.

\begin{definition}[\citet{marshall1979inequalities}] 
We say that a vector $\ba \in \R^n$ \textit{weakly majorizes} another vector $\bb \in \R^n$, denoted $\ba \succeqw \bb$, if 
\begin{equation}\label{eq:orderk}
\sum_{i=1}^k a_{(i)} \ge \sum_{i=1}^k b_{(i)}
\end{equation}
for all $1 \le k \le n$, where $a_{(1)} \ge \cdots \ge a_{(n)}$ and $b_{(1)} \ge \cdots \ge b_{(n)}$ are sorted in descending order from $\ba$ and $\bb$, respectively. If \eqref{eq:orderk}
reduces to an equality for $k = n$ while the rest $n-1$ inequalities remain the same, we say $\ba$ \textit{majorizes} $\bb$ and write as $\ba \succeq \bb$.

\end{definition}

The following lemma characterizes majorization via convex functions.

\begin{lemma}[Hardy--Littlewood--P\'olya inequality]\label{lm:jensen}
Let $\ba$ and $\bb$ be two vectors in $\R^n$. 
\begin{itemize}
\item[(a)] The inequality
\[
\sum_{i=1}^n h(a_i) \ge \sum_{i=1}^n h(b_i)
\]
holds for all nondecreasing convex functions $h$ if and only if $\ba \succeqw \bb$. 

\item[(b)] The same inequality holds for all convex functions $h$ if and only if $\ba \succeq \bb$.

\end{itemize}

\end{lemma}

\begin{remark}
This is a well-known result in theory of majorization. For a proof of Lemma~\ref{lm:jensen}, see \citet{marshall1979inequalities,arnold1987majorization}. For the proof of
Theorem~\ref{thm:compare}, however, only part (a) is needed. Part (b) will be used in the proofs of Theorems~\ref{thm:main} and \ref{thm:block}.

\end{remark}

The following lemma is instrumental to the proof of Theorem~\ref{thm:compare}. Its proof is presented later in this subsection.
\begin{lemma}\label{lm:notie}
Let $\ba$ be a vector such that each element is different. There exists $\delta > 0$ such that if $\bb_1 + \bb_2 = 2\ba$ and $\|\bb_1 - \bb_2\| < \delta$, then $\bb_1 \succeqw \ba$ and $\bb_2 \succeqw \ba$ cannot hold simultaneously unless $\bb_1 = \bb_2 = \ba$.
\end{lemma}

\proof{Proof of Theorem~\ref{thm:compare}}
Let $S$ and $S'$ be two neighboring knowledge elements in the knowledge partition $\mathcal{S}$. By assumption, the boundary between $S$ and $S'$ is a piecewise smooth surface. Pick an
arbitrary point $\bx$ on the boundary where the surface is locally smooth. Let $\epsilon > 0$ be small and $\bR = \bx + \epsilon \bv$ and $\bR' = \bx - \epsilon \bv$ for some unit-norm
vector $\bv$ that will be specified later. Assume without loss of generality that $\bR \in S$ and $\bR' \in S'$. For simplicity, we consider the noiseless setting where $\by = \bR$ and $\by' = \bR'$.

When the ground truth is $\bR$, by assumption, the author would truthfully report $S$ as opposed to $S'$. Put differently, the overall utility by reporting $S$ is higher than or equal to that
by reporting $S'$. As is evident, the mechanism would output $\by$ if the author reports $S$; if the author reports $S'$, then it would output the point, say, $\widehat \br$, that minimizes the sum
of Bregman divergences $\sum_{i=1}^n D_{\phi}(y_i, r_i)$ over the boundary between $S$ and $S'$. Assuming $\widehat\br = \bx + o(\epsilon)$ for any sufficiently small $\epsilon$ as given for the moment, we get
\[
U(\bx + \epsilon \bv) = U(\by) \ge U(\widehat\br) = U(\bx + o(\epsilon))
\]
for any nondecreasing convex function $U$. By Lemma~\ref{lm:jensen}, then, we must have $\bx + \epsilon \bv \succeqw \bx + o(\epsilon)$, from which it follows that 
\begin{equation}\nonumber
\bx + \epsilon \bv \succeqw \bx.
\end{equation}
Likewise, we can deduce
\begin{equation}\nonumber
\bx - \epsilon \bv \succeqw \bx
\end{equation}
from taking $\bR'$ as the ground truth. If each element of $\bx$ is different, Lemma~\ref{lm:notie} concludes that $\bv = \bm 0$ by taking $\bb_1 = \bx + \epsilon\bv, \bb_2 = \bx - \epsilon\bv,
\ba = \bx$, and $\epsilon$ sufficiently small. This is a contradiction.

Therefore, $\bx$ must have two entries, say, $x_i$ and $x_j$, with the same value. As $\bx$ can be an arbitrary point in the interior of any smooth surface of the boundary between $S$ and
$S'$, this shows that this surface must be part of a pairwise-comparison hyperplane.

To finish the proof, we show that, by choosing an appropriate unit-norm vector $\bv$, we will have $\widehat\br = \bx + o(\epsilon)$ for sufficiently small $\epsilon$. Note that
\[
\sum_{i=1}^n D_{\phi}(y_i, r_i) = \frac12 (\by - \br)^\top \bH_{\phi}(\br) (\by - \br) + o(\|\by - \br\|^2),
\]
where $\bH_{\phi}(\br)$ is a diagonal matrix consisting of $\phi''(r_i)$ on its diagonal for $i = 1, \ldots, n$. Owing to the twice continuous differentiability of $\phi$, this diagonal
Hessian $\bH_{\phi}(\br) = \bH_{\phi}(\bx) + o(1)$ when $\br$ is close to $\bx$. Recognizing that $\by = \bR = \bx + \epsilon \bv$ is close to the boundary when $\epsilon$ is sufficiently
small, $\widehat\br$ is the projection of $\by$ onto the tangent plane at $\bx$ under the $\bH_{\phi}(\bx)^{-1}$-Mahalanobis distance, up to low-order terms. As such, it suffices to let
$\bv$ be a normal vector to the tangent plane at $\bx$ under this Mahalanobis distance.

\endproof

\begin{remark}
The proof proceeds by taking the zero noise level. An interesting question for future investigation is to derive a possibly different necessary condition for honesty under the
assumption of a nonzero noise level.

\end{remark}

We conclude this subsection by proving Lemma~\ref{lm:notie}.
\proof{Proof of Lemma~\ref{lm:notie}}

Write $\bm\nu= \bb_1 - \ba$, which satisfies $\|\bm\nu\| < \delta/2$. Since $\ba$ has no ties, both $\bb_1$ and $\bb_2$ would have the same ranking as $\ba$ for sufficiently small
$\delta$. Without loss of generality, letting $a_1 \ge a_2 \ge \cdots \ge a_n$, we have $a_1 + \nu_1 \ge a_2 + \nu_2 \ge \cdots \ge a_n + \nu_n$ as well as $a_1 - \nu_1 \ge a_2 -
\nu_2 \ge \cdots \ge a_n - \nu_n$.

Assume that both $\bb_1 \succeqw \ba$ and $\bb_2 \succeqw \ba$. By the definition of weak majorization, this immediately gives
\[
\nu_1 \ge 0, \nu_1 + \nu_2 \ge 0, \ldots, \nu_1 + \cdots + \nu_n \ge 0
\]
and
\[
\nu_1 \le 0, \nu_1 + \nu_2 \le 0, \ldots, \nu_1 + \cdots + \nu_n \le 0.
\]
Taken together, these two displays show that $\nu_1 = \nu_2 = \cdots = \nu_n = 0$. As such, the only possibility is that $\bb_1 = \bb_2 = \ba$.

\endproof

\subsection{Proof of Theorem~\ref{thm:block}}
\label{sec:proofs-section-ref}


Write $\bI := (I_1, \ldots, I_p)$ for a coarse ranking of sizes $n_1, \ldots, n_p$. Let $\pi_{\bI, \by}$ be the permutation that sorts the entries of $\by$ in each subset $I_i$ in descending order and subsequently concatenates the $p$ subsets in order. For the first subset $I_1$, for example, it satisfies $\{ \pi_{\bI, \by}(1), \ldots, \pi_{\bI, \by}(n_1) \} = I_1$ and $\by_{\pi_{\bI, \by}(1)} \ge \by_{\pi_{\bI, \by}(2)} \ge \cdots \ge \by_{\pi_{\bI, \by}(n_1)}$. If $\by = (3.5, 7.5, 5, -1), I_1= \{1, 3\}$, and $I_2 = \{2, 4\}$, this permutation gives
\[
(\pi_{\bI, \by}(1), \pi_{\bI, \by}(2), \pi_{\bI, \by}(3), \pi_{\bI, \by}(4) ) = (3, 1, 2, 4), \quad \pi_{\bI, \by} \circ \by = (5, 3.5, 7.5, -1).
\]

When clear from the context, for simplicity, we often omit the dependence on $\by$ by writing $\pi_{\bI}$ for $\pi_{\bI, \by}$.

The proof of Theorem~\ref{thm:block} relies heavily on the following two lemmas. In particular, Lemma~\ref{lm:reduce_iso} reveals the importance of the permutation constructed above.

\begin{lemma}\label{lm:reduce_iso}
The solution to the coarse Isotonic Mechanism~\eqref{eq:isotone_incomp} is given by the Isotonic Mechanism~\eqref{eq:isotone} with $\pi = \pi_{\bI}$.
\end{lemma}

\begin{remark}
Thus, the solution to \eqref{eq:isotone_incomp} can be expressed as $\pi_{\bI}^{-1} \circ (\pi_{\bI} \circ \by)^+$.  
\end{remark}

Next, let $\bI^\star := (I_1^\star, \ldots, I^\star_p)$ be the ground-truth coarse ranking that satisfies \eqref{eq:ranking_block}, while $\bI$ is an arbitrary coarse ranking of the same sizes $n_1, \ldots, n_p$.
\begin{lemma}\label{lm:block_maj}
There exists a permutation $\rho$ of the indices $1, \ldots, n$ depending only on $\bI^\star$ and $\bI$ such that
\[
\pi_{\bI^\star} \circ (\bR + \ba) \succeqn \pi_{\bI} \circ  (\bR + \rho \circ \ba)
\]
for any $\ba \in \R^n$.
\end{lemma}

To clear up any confusion, note that $\pi_{\bI^\star} \circ (\bR + \ba) = \pi_{\bI^\star, \bR + \ba} \circ (\bR + \ba)$ and $\pi_{\bI} \circ  (\bR + \rho \circ \ba) = \pi_{\bI, \bR + \rho \circ \ba} \circ  (\bR + \rho \circ \ba)$.

The proofs of these two lemmas will be presented once we prove Theorem~\ref{thm:block} as follows.

\proof{Proof of Theorem~\ref{thm:block}}
Denote by $\widehat\bR_{\bI}$ the solution to the coarse Isotonic Mechanism \eqref{eq:isotone_incomp}. The overall utility can be written as
\[
\begin{aligned}
\util(\widehat\bR_{\bI}) &= \util(\pi_{\bI}^{-1} \circ (\pi_{\bI} \circ \by)^+) \\
&= \util((\pi_{\bI} \circ \by)^+) \\
&= \sum_{i=1}^n U((\pi_{\bI} \circ (\bR + \bz))^+_i).
\end{aligned}
\]
Since the permutation $\rho$ in Lemma~\ref{lm:block_maj} is deterministic, it follows from Assumption~\ref{ass:noise} that $\bz$ has the same distribution as $\rho \circ \bz$. This gives
\[
\E \util(\widehat\bR_{\bI}) = \E \left[ \sum_{i=1}^n U((\pi_{\bI} \circ (\bR + \bz))^+_i) \right] = \E \left[ \sum_{i=1}^n U((\pi_{\bI} \circ (\bR + \rho \circ \bz))^+_i) \right].
\]

By Lemma~\ref{lm:block_maj}, $\pi_{\bI^\star} \circ (\bR + \bz) \succeqn \pi_{\bI} \circ  (\bR + \rho \circ \bz)$. By making use of Lemma~\ref{lm:maj}, immediately, we get
\[
(\pi_{\bI^\star} \circ (\bR + \bz))^+ \succeqn (\pi_{\bI} \circ  (\bR + \rho \circ \bz))^+
\]
or, equivalently,
\begin{equation}\label{eq:coarse_h}
(\pi_{\bI^\star} \circ (\bR + \bz))^+ \succeq (\pi_{\bI} \circ  (\bR + \rho \circ \bz))^+.
\end{equation}
Next, applying Lemma~\ref{lm:jensen} to \eqref{eq:coarse_h} yields
\[
\E \util(\widehat\bR_{\bI}) =  \E \left[ \sum_{i=1}^n U((\pi_{\bI} \circ (\bR + \rho \circ \bz))^+_i) \right] \le  \E \left[ \sum_{i=1}^n U((\pi_{\bI^\star} \circ (\bR + \bz))^+_i) \right].
\]
Recognizing that the right-hand size is just the expected overall utility when the author reports the ground-truth coarse ranking, we get
\[
\E \util(\widehat\bR_{\bI}) \le \E \util(\widehat\bR_{\bI^\star}).
\]
This finishes the proof.

\endproof

\proof{Proof of Lemma~\ref{lm:reduce_iso}}
Recognizing that the constraint in \eqref{eq:isotone_incomp} is less restrictive than that of \eqref{eq:isotone} with $\pi = \pi_{\bI}$, it is sufficient to show that the minimum of \eqref{eq:isotone_incomp} also satisfies the constraint of \eqref{eq:isotone}. For notational simplicity, denote by $\widehat{\bR}$ the optimal solution to \eqref{eq:isotone_incomp}. To prove that $\pi_{\bI} \circ \widehat{\bR}$ is in descending order, it is sufficient to show that for each $i = 1, \ldots, p$, the subset $I_i$ satisfies the property that $\widehat{\bR}_{I_i}$ has the same order in magnitude as $\by_{I_i}$. 

Suppose that on the contrary $\widehat{\bR}_{I_i}$ does not have the same order in magnitude as $\by_{I_i}$ for some $1 \le i \le p$. Now let $\widehat{\bR}'$ be identical to $\widehat{\bR}$ except on the subset $I_i$ and on this subset, $\widehat{\bR}'_{I_i}$ is permuted from $\widehat{\bR}_{I_i}$ to have the same order in magnitude as $\by_{I_i}$. Note that $\widehat{\bR}'$ continues to satisfy the constraint of \eqref{eq:isotone_incomp}. However, we observe that
\[
\begin{aligned}
\|\by - \widehat{\bR}'\|^2 - \|\by - \widehat{\bR}\|^2 &= \|\by_{I_i} - \widehat{\bR}'_{I_i}\|^2 - \|\by_{I_i} - \widehat{\bR}_{I_i}\|^2 \\
&= 2\sum_{j \in I_i} y_j \widehat R_j - 2\sum_{j \in I_i} y_j \widehat R_j'.
\end{aligned}
\]
By the rearrangement inequality, we have 
\[
\sum_{j \in I_i} y_j \widehat R_j \le  \sum_{j \in I_i} y_j \widehat R_j',
\]
which concludes $\|\by - \widehat{\bR}'\|^2 \le \|\by - \widehat{\bR}\|^2$. This is contrary to the assumption that $\widehat{\bR}$ is the (unique) optimal solution to \eqref{eq:isotone_incomp}.

\endproof



\proof{Proof of Lemma~\ref{lm:block_maj}}
We prove this lemma by explicitly constructing such a permutation $\rho$. Let $\rho$ satisfy the following property: $\rho$ restricted to each subset $I_i$ is identical to $I_i^\star$ for each $i = 1, \ldots, p$ in the sense that
\[
\{\rho(j): j \in I_i\} = I_i^\star
\]
for each $i$. Moreover, for any $j \in I_i \cap I_i^\star$, we let $\rho(j) = j$, and for any other $j \in I_i \setminus I_i^\star$, we define $\rho$ to be the (unique) mapping from $I_i \setminus I_i^\star$ to $I_i^\star \setminus I_i$ such that the induced correspondence between $\bR_{I_i \setminus I_i^\star}$ and $\bR_{I_i^\star \setminus I_i}$ is nondecreasing. For example, $\rho$ maps the largest entry of $\bR_{I_i \setminus I_i^\star}$ to the largest entry of $\bR_{I_i^\star \setminus I_i}$, maps the second largest entry of $\bR_{I_i \setminus I_i^\star}$ to the second largest entry of $\bR_{I_i^\star \setminus I_i}$, and so on and so forth.


With the construction of $\rho$ in place, we proceed to prove $\pi_{\bI^\star} \circ (\bR + \ba) \succeqn \pi_{\bI} \circ  (\bR + \rho \circ \ba)$. For any $1 \le l \le n$, let $i$ satisfy $n_1 + \cdots + n_{i-1} < l \le n_1 + \cdots + n_{i-1} + n_i$ (if $l \le n_1$, then $i=1$). Now we aim to prove
\[
\sum_{j=1}^l (\bR + \ba)_{\pi_{\bI^\star}(j)} \ge \sum_{j=1}^l (\bR + \rho \circ \ba)_{\pi_{\bI}(j)},
\]
which is equivalent to
\begin{equation}\label{eq:large_i_i}
\sum_{j=1}^l R_{\pi_{\bI^\star}(j)} + \sum_{j=1}^l a_{\pi_{\bI^\star}(j)} \ge \sum_{j=1}^l R_{\pi_{\bI}(j)} + \sum_{j=1}^l a_{\rho \circ \pi_{\bI}(j)}.
\end{equation}
By the construction of $\rho$, we have
\[
\sum_{j=1}^{n_1 + \cdots + n_{i-1}} a_{\pi_{\bI^\star}(j)} = \sum_{j=1}^{n_1 + \cdots + n_{i-1}} a_{\rho \circ \pi_{\bI}(j)}.
\]
In addition, the left-hand side of \eqref{eq:large_i_i} sums over the $n_1 + \cdots + n_{i-1}$ largest entries of the true values, that is,
\[
\sum_{j=1}^{n_1 + \cdots + n_{i-1}} R_{\pi_{\bI^\star}(j)} = \sum_{j=1}^{n_1 + \cdots + n_{i-1}} R_{\pi^\star(j)},
\]
where $\pi^\star$ is the ground-truth ranking of $\bR$. Thus, it is sufficient to prove
\begin{equation}\label{eq:large_i_fi}
\sum_{j=1}^{n_1 + \cdots + n_{i-1}} R_{\pi^\star(j)} + \sum_{j=n_1 + \cdots + n_{i-1}+1}^{l} (\bR + \ba)_{\pi_{\bI^\star}(j)} \\
\ge \sum_{j=1}^l R_{\pi_{\bI}(j)} + \sum_{j=n_1 + \cdots + n_{i-1}+1}^{l} a_{\rho \circ \pi_{\bI}(j)}.
\end{equation}
Note that $J^l = \{\rho \circ \pi_{\bI}(j): n_1 + \cdots + n_{i-1}+1 \le j \le l\}$ is a subset of
\[
\{\rho \circ \pi_{\bI}(j): n_1 + \cdots + n_{i-1}+1 \le j \le n_1 + \cdots + n_{i}\} = \{\rho(j'): j' \in I_i\} = I_i^\star.
\]
Then, by the definition of $\pi_{\bI} = \pi_{\bI, \bR + \ba}$, we have
\[
\begin{aligned}
\sum_{j=n_1 + \cdots + n_{i-1}+1}^{l} (\bR + \ba)_{\pi_{\bI^\star}(j)} &= \sum_{j=n_1 + \cdots + n_{i-1}+1}^{l} (\bR + \ba)_{\rho \circ \pi_{\bI}(j)}\\
&= \sum_{j=n_1 + \cdots + n_{i-1}+1}^{l} R_{\rho \circ \pi_{\bI}(j)} + \sum_{j=n_1 + \cdots + n_{i-1}+1}^{l} a_{\rho \circ \pi_{\bI}(j)},
\end{aligned}
\]
which, together with \eqref{eq:large_i_fi}, shows that we would finish the proof of this lemma once verifying
\begin{equation}\label{eq:large_i_fi_fi}
\sum_{j=1}^{n_1 + \cdots + n_{i-1}} R_{\pi^\star(j)} + \sum_{j=n_1 + \cdots + n_{i-1}+1}^{l} R_{\rho \circ \pi_{\bI}(j)} \ge \sum_{j=1}^l R_{\pi_{\bI}(j)}.
\end{equation}

Now we prove \eqref{eq:large_i_fi_fi} as follows. By the construction of $\rho$, we have $\rho \circ \pi_{\bI}(j) = \pi_{\bI}(j)$ whenever $\pi_{\bI}(j) \in I_i \cap I_i^\star$. Since any such $\pi_{\bI}(j)$ with $n_1 + \cdots + n_{i-1}+1 \le j \le l$ contributes equally to both sides of \eqref{eq:large_i_fi_fi}, without loss of generality, we can assume that $I_i \cap I_i^\star = \emptyset$. To see why \eqref{eq:large_i_fi_fi} holds, note that if
\begin{equation}\label{eq:i_i_lar}
\sum_{j=n_1 + \cdots + n_{i-1}+1}^l R_{\pi_{\bI}(j)}
\end{equation}
is summed over the $l - n_1 - \cdots - n_{i-1}$ largest entries of $\bR_{I_i}$, then by the construction of $\rho$, 
\begin{equation}\label{eq:i_i_star_lar}
\sum_{j=n_1 + \cdots + n_{i-1}+1}^{l} R_{\rho \circ \pi_{\bI}(j)}
\end{equation}
is summed over the $l - n_1 - \cdots - n_{i-1}$ largest entries of $\bR_{I_i^\star}$. Thus, \eqref{eq:large_i_fi_fi} follows since its right-hand side is the sum of the $l$ largest entries of $\bR$. The sum \eqref{eq:i_i_lar} may skip some large entries, and \eqref{eq:i_i_star_lar} would skip correspondingly. Here, \eqref{eq:large_i_fi_fi} remains true since summation and skipping are applied to $\bR$ that has already been ordered from the largest to the smallest.

\endproof


%% file: discuss.tex
\section{Discussion}
\label{sec:discussion}

This paper introduces the Isotonic Mechanism, where a principal seeks to estimate the ground truth by eliciting information from an agent with private knowledge. Motivated by the declining quality of peer review at ML/AI conferences such as NeurIPS and ICML, our approach integrates authors' self-assessments---provided as rankings---into the raw review scores. Assuming convex utility functions for authors, we prove that the Isotonic Mechanism is truthful and enhances the accuracy of ranking-calibrated scores compared to raw scores. Moreover, we prove that it is optimal among all partition-based mechanisms, as it truthfully extracts the most fine-grained information from authors. We further explore relaxations of this mechanism, demonstrating that even with incomplete knowledge of the ground truth, authors remain truthful, and the conference can still improve estimation accuracy.

While the Isotonic Mechanism has been experimentally implemented at ICML in 2023 and 2024, and more recently at both ICML and NeurIPS in 2025, additional efforts are necessary for its adoption in decision-making processes, given that peer review at ML/AI conferences is highly sophisticated and nuanced (see Section 4 of \citet{su2021you}). First, the robustness of the Isotonic Mechanism should be analyzed under realistic conditions in which authors may inadvertently provide an incorrect ranking or possess uncertainty regarding the quality of their submissions. Another potential strategic behavior by authors could involve deliberately submitting many low-quality papers that are subsequently ranked lowest. The Isotonic Mechanism tends to elevate the scores of normal papers, thereby increasing their likelihood of acceptance. One approach to addressing this issue would be to implement a preliminary screening step to filter out exceedingly low-quality submissions. Alternatively, the review process could be designed to yield more negative utility for low-quality papers---for instance, by following ICLR's practice of making review comments public regardless of whether accepted or rejected. A further challenge arises from the fact that papers submitted to ML/AI conferences typically have multiple authors. Notably, \citet{wu2024truth} proposed an approach to applying the Isotonic Mechanism iteratively by first partitioning submissions based on the bipartite graph between authors and submissions. Ultimately, it is important to recognize that obtaining more accurate scores is not the ultimate objective but rather to select papers with high long-term scientific impact. In this regard, a recent study on the ICML 2023 ranking data demonstrated that authors' rankings are indeed a powerful predictor of future citations \citep{su2025how}. In light of these challenges and considerations, the implementation of the Isotonic Mechanism in decision-making processes should proceed cautiously and incrementally, perhaps beginning with low-risk, near-term applications such as emergency reviewer recruitment and assisting senior area chairs in overseeing area chairs' accept/reject decisions (see Section 4 of \citet{su2024analysis}).

Our work opens a plethora of avenues for future research. First, at a high level, the Isotonic Mechanism may be applied to other contexts. For instance, in player valuation, each soccer player is rated by sports performance analysis agencies such as the FIFA Index and InStat. However, the team manager possesses additional private information regarding the players' strengths and health conditions. Another potential application arises in the second-hand market: a car leasing company may sell vehicles that have been rated by a vehicle valuation agency, yet the company also holds private information about the reliability of these used cars.

Next, from a theoretical perspective, Theorem~\ref{thm:compare} shows that pairwise comparisons are necessary but not sufficient for truthfulness; a related technical question is to identify all pairwise-comparison-based knowledge partitions that are truthful. Furthermore, a pressing theoretical challenge concerns the author's utility. Relaxing the assumption of an arbitrary, unknown convex utility function may allow for a more fine-grained truthful mechanism if the utility function is known. For instance, consider the following result, with its proof deferred to the Appendix.

\begin{proposition}\label{thm:fixed_util}
Assume that $z_1, \ldots, z_n$ are i.i.d.\ random variables with mean 0 and that the overall utility is $U(\br) = \|\br\|^2$. Then the collection of all lines passing through the origin in $\mathbb{R}^n$,
\[
\mathcal{S} = \left\{ \{a \bu: a \in \R\}: \|\bu\| = 1 \right\}
\]
is a truthful knowledge partition.\footnote{A knowledge element of this partition is invariant under sign changes: $\{a \bu: a \in \mathbb{R}\} = \{a (-\bu): a \in \mathbb{R}\}$. Consequently, this knowledge partition is topologically equivalent to an $(n-1)$-sphere with antipodes identified, also known as the real projective space of dimension $n-1$. Note that this knowledge partition is not separated by piecewise smooth surfaces.}
\end{proposition}

In light of Proposition~\ref{thm:fixed_util}, a worthwhile goal is to design better truthful knowledge partitions that align with given utility functions. Moreover, a more challenging avenue of research concerns nonconvex utility functions. In the nonconvex regime, the isotonic partition is no longer truthful (see Proposition~\ref{prop:nonconvex} in the Appendix). One potential strategy for handling nonconvex utility is to add further rewards (provided, for example, by conference organizers) so that the resulting utility function is convex or approximately convex. For instance, the mechanism could include a component that slightly penalizes the author if the reported ranking is very unlikely given the observed review scores.\footnote{For example, the conference could send a warning to an author whose reported ranking is completely opposite to the ranking of the raw review scores. An interesting open question is whether such a policy could be shown to incentivize truthful ranking as the optimal strategy.}

\begin{figure}[!htp]
\centering
\begin{tikzpicture}[every text node part/.style={align=center}]
\node (One) at (-3.5,0) [shape=ellipse] {Principal}; 
\node (Two) at (3.5,-1.5) [text centered] {Owner-Assisted \\Mechanism};
\node (Three) at (10.5,0) [shape=ellipse] {Agent};
\node (Four) at (3.5,1.5) [text centered] {Ground truth};

\draw [-{Latex[scale=1]},thick,postaction={decorate,decoration={raise=1ex,text along path,text align=center,text={|\small|Template of estimator}}}] (-2.5,-0.1) to (2.1,-1.5);
\draw [-{Latex[scale=1]},dashed,postaction={decorate,decoration={raise=1ex,reverse path, text along path,text align=center,text={|\small|Noisy signal}}}] (2.2, 1.4) to [bend right=0] (-2.5,0.1);

\draw [-{Latex[scale=1]},dashed,postaction={decorate,decoration={raise=1ex,text along path,text align=center,text={|\small|Private information}}}] (4.8, 1.4) to (9.8,0.1);
\draw [-{Latex[scale=1.0]},thick,postaction={decorate,decoration={raise=1ex,reverse path,text along path,text align=center,text={|\small|Message about ground truth}}}] (9.8, -0.1) to [bend right=0]  (4.9,-1.5);

\draw [-{Latex[scale=1.0]}, dashed] (3.5, -1.1) -- (3.5,1.3) node[midway] {\small Estimation};

\end{tikzpicture}
\caption{An illustration of the general framework of owner-assisted mechanisms. The principal (conference) observes signals (raw review scores) from a statistical model parameterized by the ground truth (submission quality), while the agent (author) directly observes the ground truth. The principal determines the template for the estimator of the ground truth, which takes as input both the signals observed by the principal and a message from the agent regarding the ground truth. The principal obtains a higher payoff when the estimation of the ground truth is more accurate, whereas the agent's utility function depends on the estimator produced by the mechanism.}
\label{fig:owner}
\end{figure}

From a mechanism-design perspective, although our setting is closely related to aligned delegation \citep{frankel2014aligned} (see Section~\ref{sec:related}), the additional channel of the principal (conference) observing the review scores provides an opportunity to incorporate statistical estimation into the delegation problem. As illustrated in Figure~\ref{fig:owner}, in this owner-assisted setting, the agent's action (i.e., the author's ranking) is not the final outcome; rather, it is an input to the mechanism, which in turn outputs an estimate of the ground truth. This structure enables flexibility in designing the form of the estimator used in the mechanism. For example, it would be worthwhile to investigate alternative estimation approaches for constructing truthful mechanisms or to consider other statistical models in which the raw scores are generated from the ground truth under different noise distributions. Notably, \citet{yan2023isotonic} extended the Isotonic Mechanism to settings where the scores follow exponential family distributions. Such extensions are numerous and open-ended, and developing a unified framework to optimize mechanisms over all possible estimation approaches remains a direction for future work.


%% file: appendix.tex
\section*{Appendix}
\label{sec:appendix}

In the Appendix, we provide the proofs of technical details that are omitted in the main text.

\proof{Proof of Proposition~\ref{thm:estimate}}
Consider the (possibly degenerate) triangle formed by $\by, \bR, \widehat\bR^S$. Assuming the angle $\measuredangle(\by, \widehat\bR^S, \bR) \ge 90^\circ$ for the moment, we immediately conclude that $\|\by - \bR\| \ge \|\widehat\bR^S - \bR\|$, thereby proving the proposition. To finish the proof, suppose the contrary that $\measuredangle(\by, \widehat\bR^S, \bR) < 90^\circ$. Then there must exist a point $\bR'$ on the segment between $\widehat\bR^S$ and $\bR$ such that $\|\by - \bR'\| < \|\by - \widehat\bR^S\|$. Since both $\widehat\bR^S$ and $\bR$ belong to the (convex) isotonic cone $\{\bx: x_{\pi^\star(1)} \ge \cdots \ge x_{\pi^\star(n)}\}$, the point $\bR'$ must be in the isotonic cone as well. However, this contradicts the fact that $\widehat\bR^S$ is the (unique) point of the isotonic cone with the minimum distance to $\by$.

\endproof

\proof{Proof of Proposition~\ref{prop:n=1}}

Assume that $\mathcal S$ is a nontrivial knowledge partition. Pick any knowledge element $S \in \mathcal S$ and let $x$ be an interior point of $S$. Consider the noiseless setting with utility function $U(x) = x$, which is a nondecreasing convex function. Write $a = \sup_y\{y: \text{interval } (x, y) \subset S\}$. If $a = \infty$, then $S$ contains all sufficiently large numbers. In this case, we instead pick any different knowledge element in order to ensure $a < \infty$.

Therefore, we can assume $a < \infty$. Let $S'$ be the knowledge element that contains a (small) right neighborhood of $a$. Taking ground truth $R = \frac{x + 2a}{3} < a$, if the author reports $S'$, then the solution would be $a$. Since $U(a) > U(R)$, the author would be better off reporting $S'$ instead of $S$. This contradiction demonstrates that $\mathcal S$ must be trivial.

\endproof

\proof{Proof of Proposition~\ref{prop:improve}}

Recall that $\widehat\bR^S$ is the solution to
\[
\begin{aligned}
&\min_{\br} ~ \| \by - \br\|^2 \\
&\text{~s.t.} ~~ \br \in S.
\end{aligned}
\]
When the noise level in $\by = \bR + \bz$ tends to zero and the ground truth $\bR \in S$, the projection of $\by$ onto $S$ is asymptotically equivalent to the projection of $\by$ onto the tangent cone of $S$ at $\bR$ (see \citet{rockafellar2015convex}). More precisely, letting $T_S(\bR)$ be the tangent cone of $S$ at $\bR$ and writing $\widehat\bR^{T_S(\bR)}$ for the projection of $\by$ onto $T_S(\bR)$, we have $\widehat\bR^S = \widehat\bR^{T_S(\bR)} + o(\|\widehat\bR^S - \bR\|)$. This fact implies that, with probability tending to one,
\[
\limsup_{\sigma \goto 0} \frac{\|\widehat\bR^{T_S(\bR)} - \bR\|^2}{\|\widehat\bR^S - \bR\|^2} = \limsup_{\sigma \goto 0} \frac{\E \|\widehat\bR^{T_S(\bR)} - \bR\|^2}{\E \|\widehat\bR^S - \bR\|^2} = 1.
\]
To prove the first part of Proposition~\ref{prop:improve}, therefore, it suffices to show that
\[
\|\widehat\bR^{T_{S_2}(\bR)} - \bR\|^2 \le \|\widehat\bR^{T_{S_1}(\bR)} - \bR\|^2
\]
with probability one. This inequality follows from the fact that $T_{S_2}(\bR) \subset T_{S_1}(\bR)$ and both cones have apex at $\bR$.

Next, we prove the second part. Because both $S_1$ and $S_2$ are cones, it follows from Moreau's decomposition theorem that
\[
\|\widehat\bR^{S_1} - \by\|^2 + \|\widehat\bR^{S_1}\|^2 = \|\by\|^2
\]
and
\[
\|\widehat\bR^{S_2} - \by\|^2 + \|\widehat\bR^{S_2}\|^2 = \|\by\|^2.
\]
Since $S_2 \subset S_1$, we get $\|\widehat\bR^{S_1} - \by\|^2 \le \|\widehat\bR^{S_2} - \by\|^2$, which in conjunction with the two identities above gives
\begin{equation}\label{eq:s1s2}
\|\widehat\bR^{S_1}\|^2 \ge \|\widehat\bR^{S_2}\|^2.  
\end{equation}
In the limit $\sigma \goto \infty$, we have $\|\widehat\bR^{S_1} - \bR\|^2 = (1 + o(1))\|\widehat\bR^{S_1}\|^2$ and $\|\widehat\bR^{S_2} - \bR\|^2 = (1 + o(1))\|\widehat\bR^{S_2}\|^2$ with probability tending to one. Together with \eqref{eq:s1s2}, this concludes
\[
\limsup_{\sigma \goto \infty} \frac{\E \|\widehat\bR^{S_2} - \bR\|^2}{\E \|\widehat\bR^{S_1} - \bR\|^2} \le \limsup_{\sigma \goto \infty} (1 + o(1)) = 1.
\]

\endproof

\proof{Proof for the example in Section~\ref{sec:exampl-truth-tell}}

First, consider the case $U(x) = x^2$. For simplicity, we start by assuming $\bR = \bm 0$. Due to symmetry, the expected overall utility of reporting $S_1$ is the same as that of reporting an arbitrary isotonic cone $S_{\pi}$. In particular, taking any $S_{\pi} \ne S_1$, we have $S_{\pi} \subset S_2 $. The proof of Proposition~\ref{prop:improve} above shows that 
\begin{equation}\label{eq:lager_coned}
\|\widehat\bR^{S_\pi}\|^2  \le \|\widehat\bR^{S_2}\|^2
\end{equation}
with probability one. By taking the expectation, we get
\[
\E U(\widehat\bR^{S_\pi}) = \E \|\widehat\bR^{S_\pi}\|^2  < \E \|\widehat\bR^{S_2}\|^2 = \E U(\widehat\bR^{S_2})
\]
since \eqref{eq:lager_coned} is a strict inequality with positive probability. Equivalently, we get
\begin{equation}\label{eq:some_thing}
\E U(\widehat\bR^{S_1}) < \E U(\widehat\bR^{S_2}).
\end{equation}
Moving back to the case $\bR = (n\epsilon, (n-1)\epsilon, \ldots, 2\epsilon, \epsilon) \in S_1$, \eqref{eq:some_thing} remains valid for sufficiently small $\epsilon$.

Next, we consider $U(x) = \max\{0, x\}^2$. As earlier, we first assume $\bR = \bm 0$. For any isotonic cone $S_{\pi}$, let us take as given for the moment that the empirical distribution of the entries of $\widehat\bR^{S_\pi}$ is symmetric with respect to the origin over the randomness of the Gaussian noise. This symmetry is also true for $S_2$. Therefore, we get
\[
\E U(\widehat\bR^{S_1}) = \frac12 \E \|\widehat\bR^{S_1}\|^2 < \frac12 \E \|\widehat\bR^{S_2}\|^2 = \E U(\widehat\bR^{S_2}).
\]
This inequality continues to hold for sufficiently small $\epsilon$.

To finish the proof, we explain why the above-mentioned symmetric property of $\widehat\bR^{S_\pi}$ in distribution is true. Let $\pi^-$ be the reverse ranking of $\pi$, that is, $\pi^-(i) = \pi(n +1 - i)$ for all $1 \le i \le n$. For any Gaussian noise vector $\bz = (z_1, \ldots, z_n)$, it is easy to see that the entries (as a set) of the projection of $\bz$ onto $S_\pi$ are negative to the entries (as a set) of the projection of $- \pi^- \circ \bz$ onto $S_\pi$. Finally, note that $- \pi^- \circ \bz$ has the same probability distribution as $\bz$. This completes the proof.

\endproof

\proof{Proof of Proposition~\ref{thm:incomplete}}
Recognizing that the solution to the Isotonic Mechanism takes the form $\pi^{-1} \circ (\pi \circ (\bR + \bz))^+$, the expected overall utility is
\[
\begin{aligned}
\E \util(\pi^{-1} \circ (\pi \circ (\bR + \bz))^+) &= \E \util((\pi \circ (\bR + \bz))^+)\\
&= \E \util((\pi \circ \bR + \pi \circ \bz)^+)\\
&= \E \util((\pi \circ \bR + \bz)^+)\\
&= \E \left[ \sum_{i=1}^n U((\pi \circ \bR + \bz)^+_i) \right],
\end{aligned}
\]
where we use the exchangeability of the distribution of the noise vector $\bz$. Next, the assumption that $\pi_1$ is more consistent than $\pi_2$ with respect to the ground truth $\bR$ implies
\[
\pi_1 \circ \bR \succeqn \pi_2 \circ \bR,
\]
from which it follows that $\pi_1 \circ \bR + \bz \succeqn \pi_2 \circ \bR + \bz$. By the Hardy--Littlewood--P\'olya inequality, therefore, we get
\[
\sum_{i=1}^n U((\pi_1 \circ \bR + \bz)^+_i)  \ge \sum_{i=1}^n U((\pi_2 \circ \bR + \bz)^+_i) 
\]
for all $\bz$. This concludes
\[
\E \util(\pi_1^{-1} \circ (\pi_1 \circ (\bR + \bz))^+) \ge \E \util(\pi_2^{-1} \circ (\pi_2 \circ (\bR + \bz))^+).
\]

\endproof

\proof{Proof of Lemma~\ref{lm:mono}}
  
The proof relies on the min-max formula of isotonic regression (see Chapter 1 of \citet{minimax}):
\[
a^+_k = \max_{v \ge k} \min_{u \le k} \frac{a_u + a_{u+1} + \cdots + a_v}{v - u + 1}
\]
for any $k = 1, \ldots, n$. For simplicity, write $\bb = \ba + \delta \bm\e_i$. Then it is clear that $\bb$ is larger than or equal to $\ba$ in the component-wise sense. Therefore, we get
\[
b^+_k = \max_{v \ge k} \min_{u \le k} \frac{b_u + b_{u+1} + \cdots + b_v}{v - u + 1} \ge \max_{v \ge k} \min_{u \le k} \frac{a_u + a_{u+1} + \cdots + a_v}{v - u + 1} = a^+_k
\]
for all $k = 1, \ldots, n$. This concludes the proof.

\endproof


\proof{Proof of Lemma~\ref{lm:single_val}}
To begin with, assume that $\ba^+$ has constant entries. Suppose the contrary that 
\[
\bar a_k :=\frac{a_1 + \cdots + a_k}{k} > \bar a
\]
for some $k$, which implies
\[
\frac{a_1 + \cdots + a_k}{k} > \bar a > \bar a_{-k} : = \frac{a_{k+1} + \cdots + a_n}{n - k}.
\]
This inequality allows us to get
\[
\begin{aligned}
&\sum_{i=1}^k (a_i - \bar a_k)^2 + \sum_{i=k+1}^n (a_i - \bar a_{-k})^2  \\
&= \sum_{i=1}^k (a_i - \bar a)^2 - k (\bar a_k - \bar a)^2 + \sum_{i=k+1}^n (a_i - \bar a)^2  - (n-k) (\bar a_{-k} - \bar a)^2\\
& <  \sum_{i=1}^n (a_i - \bar a)^2\\
& = \|\ba - \ba^+\|^2.
\end{aligned}
\]
As such, the vector formed by concatenating $k$ copies of $\bar a_k$ followed by $n-k$ copies of $\bar a_{-k}$, which lies in the standard isotonic cone since $\bar a_k > \bar a_{-k}$, leads to a smaller squared error than $\ba^+$. This contradicts the definition of $\ba^+$.

Next, we assume that 
\begin{equation}\label{eq:piece_bal}
\frac{a_1 + \cdots + a_k}{k} \le \bar a
\end{equation} 
for all $k =1, \ldots, n$. To seek a contradiction, suppose that $\ba^+$ has more than one constant piece. That is, a partition of $\{1, 2, \ldots, n\} = I_1 \cup I_2 \cup \cdots \cup I_L$ with $L \ge 2$ from the left to the right satisfies the following: the entries of $\ba^+$ on $I_l$ are constant and, denoting by $a^+_{I_l}$ the value on $I_l$, the isotonic constraint requires 
\begin{equation}\label{eq:dec_piece}
a^+_{I_1} > a^+_{I_2} > \cdots > a^+_{I_L}.
\end{equation} 
Making use of a basic property of isotonic regression, we have
\[
a^+_{I_l} = \frac{\sum_{i \in I_l} a_i}{|I_l|}
\]
for $l = 1, \ldots, L$, where $|I_l|$ denotes the set cardinality. This display, together with \eqref{eq:dec_piece}, shows
\[
\begin{aligned}
\bar a &= \frac{a_1 + \cdots + a_n}{n} \\
&= \frac{|I_1| a^+_{I_1} + |I_2| a^+_{I_2} + \cdots + |I_L| a^+_{I_L}}{n} \\
&< \frac{|I_1| a^+_{I_1} + |I_2| a^+_{I_1} + \cdots + |I_L| a^+_{I_1}}{n} \\
&= a^+_{l_1}\\
& = \frac{a_1 + \cdots + a_{|I_1|}}{|I_1|},
\end{aligned}
\]
a contradiction to \eqref{eq:piece_bal}.

\endproof

\proof{Proof of Proposition~\ref{thm:fixed_util}}

Denote by $\bu$ the unit-norm vector reported by the author (equivalently, $-\bu$). The output of the mechanism is 
\[
\widehat\bR^{\bu} = (\bu \cdot \by) \bu = (\bu \cdot (\bR + \bz)) \bu = (\bu \cdot \bR) \bu + (\bu \cdot \bz) \bu.
\]
The overall utility is
\[
\|\widehat\bR^{\bu}\|^2 = \|(\bu \cdot \bR) \bu + (\bu \cdot \bz) \bu\|^2.
\]
Its expectation is
\[
\E\|(\bu \cdot \bR) \bu + (\bu \cdot \bz) \bu\|^2 = \|(\bu \cdot \bR) \bu\|^2 + \E \|(\bu \cdot \bz) \bu\|^2 + 2\E \left[ \left((\bu \cdot \bR) \bu\right) \cdot \left( (\bu \cdot \bz) \bu \right)\right].
\]
We have $\|(\bu \cdot \bR) \bu\|^2 = (\bu \cdot \bR)^2\|\bu\|^2 = (\bu \cdot \bR)^2$ and
\[
\begin{aligned}
\E \|(\bu \cdot \bz) \bu\|^2 = \E (\bu \cdot \bz)^2 \|\bu\|^2 = \E (\bu \cdot \bz)^2 = \|\bu\|^2 \E z_1^2 = \E z_1^2,
\end{aligned}
\]
where the third equality makes use of the fact that $z_1, \ldots, z_n$ are i.i.d.~centered random variables. Besides, we have
\[
2\E \left[ \left((\bu \cdot \bR) \bu\right) \cdot \left( (\bu \cdot \bz) \bu \right)\right] = 2 \left((\bu \cdot \bR) \bu\right) \cdot \left( (\bu \cdot \E \bz) \bu \right) = 0.
\]
Thus, we get
\[
\E \|\widehat\bR^{\bu}\|^2 = (\bu \cdot \bR)^2 + \E z_1^2 \le \|\bR\|^2 + \E z_1^2,
\]
with equality if and only if $\bu$ has the same direction as $\bR$, that is, $\bR \in \{a \bu: a \in \R\}$. In words, the author would maximize her expected overall utility if and only if she reports the line that truly contains the ground truth.

\endproof

\begin{proposition}\label{prop:nonconvex}
Under Assumptions~\ref{ass:author2} and \ref{ass:noise}, if the utility function $U$ in \eqref{eq:util_form} is nonconvex, then there exists a certain ground truth $\bR$ and a noise distribution such that the author is not truthful under the Isotonic Mechanism.

\end{proposition}

\proof{Proof of Proposition~\ref{prop:nonconvex}}
Let the noise vector $\bz = \bm 0$. Since $U$ is not convex, there must exist $r_1 > r_2$ such that
\begin{equation}\label{eq:r1r2_non}
U(r_1) + U(r_2) < 2 U\left( \frac{r_1 + r_2}{2} \right).
\end{equation}
Let the ground truth $\bR$ satisfy $\bR_1 = r_1, \bR_2 = r_2$, and $\bR_{i} = r_2 - i$ for $i = 3, \ldots,n$. Note that $\bR$ is in descending order. If the author reports the ground-truth ranking, the solution to the Isotonic Mechanism is $\bR$ itself and her overall utility is
\begin{equation}\label{eq:r_util1}
U(r_1) + U(r_2) + \sum_{i=3}^n U(r_2 - i).
\end{equation}
However, if the author reports $\pi$ such that $\pi(1) = 2, \pi(2) = 1$, and $\pi(i) = i$ for $i \ge 3$, then the solution with this ranking is 
\[
\widehat\bR^{\pi} = \left(\frac{r_1 + r_2}{2}, \frac{r_1 + r_2}{2}, r_2 - 3, r_2 - 4, \ldots, r_2-n \right).
\]
The corresponding overall utility is
\begin{equation}\label{eq:r_util2}
U\left( \frac{r_1 + r_2}{2} \right) + U\left( \frac{r_1 + r_2}{2} \right) + \sum_{i=3}^n U(r_2 - i) = 2U\left( \frac{r_1 + r_2}{2} \right) + \sum_{i=3}^n U(r_2 - i).
\end{equation}
It follows from \eqref{eq:r1r2_non} that \eqref{eq:r_util2}$>$\eqref{eq:r_util1}, thereby implying that the author would be better off reporting the incorrect ranking $\pi$ instead of the ground-truth ranking.


\endproof
